%% file: main.tex
\pdfoutput=1
\documentclass[11pt, letterpaper]{article}
\input{macros.tex}

\title{One-Shot Private Confidence Regions via Resampling}
\date{}
\author{
\begin{tabular}{c@{\hspace{2em}}c}
Shourya Pandey & Purnamrita Sarkar \\
{\small University of Texas at Austin} &
{\small University of Texas at Austin} \\
{\small\texttt{shouryap@utexas.edu}} &
{\small\texttt{purna.sarkar@austin.utexas.edu}}
\\[1.5em]
Po-Ling Loh & Debepsita Mukherjee \\
{\small University of Cambridge} &
{\small University of Texas at Austin} \\
{\small\texttt{pll28@dpmms.cam.ac.uk}} &
{\small\texttt{debepsitamukherjee@utexas.edu}}
\end{tabular}
}


\begin{document}
\maketitle\footnote{SP and PS gratefully acknowledge NSF grants 2217069 and CCF-2505865.}

\begin{abstract}
We propose a simple framework for constructing differentially private confidence regions \textit{in one shot}, i.e., by adding noise only to the final resampling quantile instead of privatizing the estimator computed on each resample. The cost of privacy of our procedure is only logarithmic in the number of resamples $B$ under with-replacement (m-out-of-n) sampling and independent of $B$ under without replacement sampling (subsampling), avoiding the $\sqrt{B}$ factor that arises in previous works. We provide \textit{nonasymptotic Gaussian Differential Privacy} (GDP) and utility guarantees for both subsampling and $m$-out-of-$n$ resampling, covering mean-like estimators with small global sensitivity as well as estimators admitting efficiently computable \textit{smooth sensitivity} bounds, including quantiles and degenerate U-statistics. This allows us to also obtain private confidence regions for degenerate U-statistics where the private error is much smaller than the non-private error. In all, we provide a toolbox for widely applicable DP uncertainty quantification procedures under popular resampling strategies while avoiding the computational and privacy costs of privatizing many intermediate resample statistics.

\end{abstract}


\newpage 

\begingroup
\footnotesize
\tableofcontents
\endgroup

\newpage

\input{intro}
\input{related}
\input{Notation/_notation}

\input{Algorithms/_algorithms}
\input{experiments}
\input{conclusion}

\newpage


\bibliographystyle{alpha}
\bibliography{refs}

\FloatBarrier
\appendix


\include{appendix}

\end{document}

%% file: macros.tex
\usepackage{fullpage}
\usepackage{blindtext}
\usepackage{hyperref}
\hypersetup{
	colorlinks=true,
	linkcolor=blue!70!black,
	citecolor=blue!70!black,
	urlcolor=blue!70!black
}
\usepackage{booktabs,tabularx,array}
\newcolumntype{L}[1]{>{\raggedright\arraybackslash}p{#1}}
\newcolumntype{Y}{>{\raggedright\arraybackslash}X}
\usepackage[english]{babel}
\usepackage[T1]{fontenc}
\usepackage{natbib}
\usepackage{tablefootnote}
\usepackage{placeins}
\usepackage{multirow}
\usepackage{tablefootnote}
\usepackage[table]{xcolor}
\usepackage{amsmath}
\usepackage{todonotes}
\usepackage{amssymb}
\usepackage{amsthm}
\usepackage{accents}
\usepackage{algorithm}
\usepackage{subcaption,makecell}
\usepackage[lined,boxed,ruled,norelsize,algo2e]{algorithm2e}
\usepackage{bbm}
\usepackage{bbold}
\usepackage{bm}
\usepackage{graphicx}
\usepackage{subcaption}
\graphicspath{{./figs/}}

\usepackage{cleveref}
\usepackage{thm-restate}
\usepackage{algorithm}
\usepackage{algpseudocode}

\newtheorem{theorem}{Theorem}[section]

\newtheorem{lemma}[theorem]{Lemma}
\newtheorem{fact}[theorem]{Fact}
\newtheorem{definition}[theorem]{Definition}

\newtheorem{proposition}[theorem]{Proposition}

\newtheorem{assumption}[theorem]{Assumption}
\newtheorem{remark}[theorem]{Remark}

\newcommand{\qnm}{Q_{n,m}}

\newcommand{\qnmp}{\widetilde{Q}_{n,m}}

\renewcommand{\epsilon}{\varepsilon}

\newcommand{\norm}[1]{\left\lVert#1\right\rVert}
\newcommand{\norms}[1]{\lVert#1\rVert}

\newcommand{\eps}{\varepsilon}

\newcommand{\iid}{\textup{iid}} 
\newcommand{\R}{\mathbb{R}}

\newcommand{\bas}[1]{\begin{align*}#1\end{align*}}
\newcommand{\ba}[1]{\begin{align}#1\end{align}}

\newcommand{\bk}{\color{black}}
\newcommand{\rd}{\color{red}}

\newcommand{\E}{\mathbb{E}}

\newcommand{\id}{\mathbf{I}}

\newcommand{\bb}[1]{\left(#1\right)}

\usepackage{xcolor}
\newcommand{\purna}[1]{\textcolor{green}{\textbf{PS:} #1}}

\renewcommand{\Pr}{\mathbb{P}}

\SetKw{Break}{break}
\SetKw{Return}{return}
\SetKw{KwTrue}{True}
\SetKw{KwFalse}{False}

\newcommand{\WR}{\textsc{WR}}
\newcommand{\WOR}{\textsc{WOR}}
\newcommand{\WoR}{\textsc{WOR}}
\newcommand{\kstar}{\kappa_{\star}}

\newcommand{\hatn}{\widehat{\theta}_n}
\newcommand{\tiln}{\widetilde{\theta}_n}

\newcommand{\LS}{\textsc{LS}}

\renewcommand{\SS}{\textsc{SS}}

\newcommand{\cX}{\mathcal{X}}
\newcommand{\cY}{\mathcal{Y}}

\newcommand{\mprob}{\operatorname{\mathbb{P}}}

\newcommand{\dpgs}{\texttt{CR-GS}}
\newcommand{\dpss}{\texttt{CR-SS}}

\newcommand{\TV}{\operatorname{TV}}

%% file: intro.tex
\section{Introduction}
Consider the problem of parameter estimation: given a dataset $D$ of $n$ i.i.d. samples from an unknown distribution $\mathcal{D}$ and a functional $\theta(.)$, the goal is to design an estimator for the target parameter $\theta(\mathcal{D})$. Uncertainty quantification goes beyond point estimates by providing valid confidence regions for the parameter of interest.
A classical approach treats the observed dataset as a proxy for the underlying distribution, and uses it to generate many auxiliary datasets whose distributions approximate that of the original sample. The estimator is then applied to each auxiliary dataset, producing an ensemble of parameter estimates $\{\hat{\theta}_n^*\}$. These methods are broadly known as \textit{resampling} methods, a standard and influential topic in statistics, such as the Quenouille-Tukey Jackknife~\citep{quenouille1949approximate,tukey1958bias}, bootstrap~\citep{efron1979bootstrap}, subsampling~\citep{politis1999subsampling}, and $m$-out-of-$n$ bootstrap~\citep{RESAMPLING}.

In particular, many classical confidence sets can be expressed in terms of quantiles of a suitably scaled estimation error. Resampling provides a generic plug-in approach for estimating these quantiles, avoiding the need for analytic variance formulas or asymptotic normality assumptions tailored to a specific model. Now, the nonparametric bootstrap can be computationally intensive since it requires drawing many size-$n$ with-replacement resamples from the size-$n$ dataset. Moreover, there are several settings in which the bootstrap is not consistent~\citep{athreya1987bootstrap,BI_andrews_2000}. Roughly speaking, the bootstrap works  when $\sqrt{n}(\hat{\theta}_n-\theta)$ is asymptotically Gaussian (see~\cite{Vaart_1998} for sufficient conditions). A more general procedure is subsampling~\citep{politis1999subsampling} and the $m$-out-of-$n$ bootstrap~\citep{RESAMPLING}. These methods sample $m$ points from the dataset of size $n$, with or without replacement, where $m=o(n)$ and $m = \omega(1)$. They are valid as long as $\tau_n(\hat{\theta}_n-\theta)$ converges to some nondegenerate asymptotic distribution, for an appropriate scaling factor $\tau_n$. 

In this paper, we study \textit{uncertainty quantification under differential privacy (DP)}, a problem that has received growing attention.~\cite{Karwa2017FiniteSD} provides one of the first private confidence intervals for the mean under the Gaussian model and proves lower bounds on the expected length of any $(1-\alpha)$-level $(\varepsilon,\delta)$-DP confidence set. A central challenge for DP resampling is that releasing statistics from many resamples incurs privacy loss that grows with the number of resamples $B$ through composition. 

Several recent works address this issue through variants of private resampling.~\cite{chadha2024resampling} uses the Bag of Little Bootstraps (BLB)~\citep{kleiner2014scalable} on disjoint data blocks to control the sensitivity, but still privatizes each bootstrap replicate. They then apply a noisy median/order-statistic mechanism together with sparse vector/AboveThreshold~\citep{dworkBook2014} to search over a grid of candidate regions and return the first region achieving the target coverage. For statistics converging at the $1/\sqrt{n}$ rate, this approach requires at least $\sqrt{n}$ candidate intervals, increasing the computational cost by a factor of $\sqrt{n}$.~\cite{Dette2025GaussianDP} analyzes $m$-out-of-$n$ resampling under GDP and obtains an effective privacy loss scaling as $\mu_{\mathrm{tot}}\approx \mu_B \frac{m}{n}\sqrt{B}$. Their privacy guarantees are \textit{asymptotic} and rely on the Central Limit Theorem (CLT) and composition arguments of~\cite{dong2022gaussian} together with the mixture-of-mechanisms framework of~\cite{WangChengAwan2025DPBootstrap}.~\cite{shoham2025black} obtains $(\epsilon,\delta)$ private CIs by running a \emph{private} estimator on $B$ subsamples and combining the privacy loss via subsampling amplification plus (basic/advanced) composition (also incurring a $\sqrt{B}$-type dependence). In contrast, we add noise only once to the final resampling quantile/radius, yielding a privacy loss that grows only \emph{logarithmically} in $B$. 

\paragraph{How large is $B$?}
In practice, the bootstrap uses a Monte Carlo approximation based on $B$ resamples; by the DKW inequality,
$\sup_{t\in\mathbb R}\big|\widehat F_B(t)-F^\ast(t)\big|=O_p(B^{-1/2})$, so bootstrap quantiles converge as $B\to\infty$ \citep{Vaart_1998}.
Empirically, one typically takes $B\approx 1000$ or $2000$ for ``rough'' CI's, and $B\approx 10^4$ to make Monte Carlo error negligible~\citep{hesterberg2014teachers,efron1979bootstrap}.  In the private resampling setting, for each of $s$ disjoint blocks of size $m$,
~\cite{chadha2024resampling} use
\[B=\min\Big\{10{,}000,\ \max\Big\{100,\ \frac{n^{1.5}}{\log n}\Big\}\Big\}.\]

While distribution-privatizing methods can release many quantiles, standard resampling CIs typically use only one or two to set the endpoints. This is the motivation of our work.

\subsection{Our contributions}
We consider two resampling schemes and two classes of statistics for which we provide confidence regions for the relevant parameter. \\

\textbf{Resampling schemes.}
Two subsampling methods are particularly relevant:
\begin{itemize}
    \item \textit{Sampling without replacement ($\WOR$)}~\citep{politis1999subsampling} Each set $S_i$ is drawn uniformly from the $\binom{n}{m}$ $m$-element subsets of $[n]$.
    
    \item \textit{Sampling with replacement ($\WR$)}~\citep{RESAMPLING} Each multiset $S_i$ is formed by drawing $m$ (not-necessarily-distinct) indices independently and uniformly from $[n]$.
\end{itemize}

We consider two broad regimes of estimators for private resampling. Let $D=(X_1,\dots,X_n)$ denote a dataset. Let $\theta(.)$ be a functional that can be applied to the population distribution to obtain an unknown parameter of interest, or applied to the data to obtain an estimator.
For an estimator $\theta:\mathcal{X}^n \to \mathbb{R}^d$, define its (global) $\ell_2$-sensitivity as
\[
\Delta_2(\theta)\;:=\;\sup_{D\sim D'} \bigl\|\theta(D)-\theta(D')\bigr\|_2,
\]
where $D\sim D'$ denotes neighboring datasets that differ in exactly one record.
\begin{definition}[Mean-like estimators] \label{def:gs-bounded}
We call $\theta$ \emph{mean-like} if $\Delta_2(\theta)\le \ell_n/n$ for some sequence $\{\ell_n\}_{n \ge 1}$
(typically, $\ell_n = C$ for a constant $C$ or $\ell_n=\mathrm{polylog}(n)$).
\end{definition}

In DP literature, one often captures statistics with large (possibly unbounded) global sensitivity by using local sensitivity, defined as
\ba{
\LS(\theta;D)\;:=\;\sup_{D'\sim D}\bigl\|\theta(D)-\theta(D')\bigr\|_2. \label{eq:definition_local_sensitivity}
}
Since the local sensitivity can itself leak privacy, one may instead compute a $\beta$-smooth upper bound  $\SS_{\beta}(\theta;\cdot):\mathcal{X}^n\to\mathbb{R}_+$. This function upper-bounds the local sensitivity and varies smoothly across neighbors:
\ba{
\LS(\theta;D)\le \SS_{\beta}(\theta;D)
\qquad\text{and}\qquad
\SS_{\beta}(\theta;D)\le e^{\beta}\,\SS_{\beta}(\theta;D')\ \ \text{for all }D\sim D' \label{eq:definition_smooth_sensitivity}
}
\citep{nissim2007smooth}. Equivalently, the map $D\mapsto \log \SS_{\beta}(\theta;D)$ has global sensitivity at most $\beta$.

\begin{definition}[$\textup{S}$-friendly estimators] \label{def:ss-bounded}
We call an estimator $\theta$ \emph{smooth-sensitivity-friendly} if there exists a parameter $\beta>0$ and an efficiently computable
$\beta$-smooth sensitivity upper bound $\SS_{\beta}(\theta;\cdot)$ such that, with high probability over the data,
$\SS_{\beta}(\theta;D)$ is small even when $\Delta_2(\theta)$ is large (or unbounded). 
\end{definition}

We pose the following question for both \textbf{mean-like} (Definition~\ref{def:gs-bounded}) and \textbf{S-friendly}  (Definition~\ref{def:ss-bounded}) estimators:

\textit{Can we design a differentially private confidence interval for a fixed coverage level, based on subsampling or $m$-out-of-$n$ resampling, that uses $B$ resamples for statistical accuracy while incurring only \emph{logarithmic} growth of privacy loss in $B$?}

To answer the above question, we propose a novel approach: rather than making each resample private or performing a private iterative search, we add noise \emph{directly to the resampling quantile} that determines the confidence set. 
 A key technical point is that, under sampling with replacement, the sensitivity of the resampling-based quantile is governed by a single overlap quantity: the \emph{maximum} number of times a modified data point can reappear across the resamples. 
Inspired by this, we propose a simple (white-box) alteration to the resampling methods that yields the same utility with a stronger privacy guarantee. The advantage of this method is that, for a large class of statistics, it yields \textit{nonasymptotic} GDP privacy guarantees.


Our contributions are as follows:
\begin{itemize}
  \item We provide an algorithm for privatizing confidence regions (CR) resulting from both the subsampling and the $m$-out-of-$n$ resampling schemes. Our method only leads to  privacy loss logarithmic in $B$.
  \item Our approach is most effective in settings where the estimator has small global sensitivity (Definition~\ref{def:gs-bounded}), or admits a small (data-dependent) smoothed sensitivity bound (Definition~\ref{def:ss-bounded}). This enables constructing private CRs for a broad class of mean-like functionals including $U$-statistics and empirical risk minimizers for strongly convex losses, as well as nonstandard targets such as quantiles and degenerate $U$-statistics.
 Degenerate $U$-statistics are a canonical example where classical asymptotic normality may fail, and the usual nonparametric (Efron) bootstrap can be inconsistent~\citep{Bretagnolle1983,arcones_gine1992}. In such cases, consistent inference can still be obtained via $m$-out-of-$n$ resampling or subsampling, which motivates our work.

  \item We give utility guarantees that separate the non-private resampling error from the privacy-induced error and discuss how the privacy term compares to prior private bootstrap approaches that compose across $B$ resamples or rely on DP threshold-search procedures.
\end{itemize}

\begin{table}[t]
\centering
\small
\renewcommand{\arraystretch}{1.2}
\setlength{\tabcolsep}{3.5pt}
\begin{tabular}{@{}l|l|l|l|l|l@{}}
\toprule
\textbf{Method}
& \textbf{Mean-like}
& \makecell{\textbf{Quantiles}\\ ($d=1$)}
& \makecell{\textbf{Degenerate}\\\textbf{$U$-stats ($d=1$)}}
& \makecell{\textbf{\#}\\\textbf{Resamples}}
& \makecell{\textbf{Nonasymp.}\\\textbf{Privacy?}} \\
\midrule
\textbf{Scale} $\tau_n$ & $\sqrt{n}$&$\sqrt{n}$& $n$&---
& --- \\[1.25ex]\hline
\textbf{Nonprivate sq. radius $r^2$}
& $\displaystyle \frac{q_{1-\alpha}^2}{n}+O\!\left(\frac{1}{nB}\right)$
& $\displaystyle \frac{q_{1-\alpha}^2}{n}+O\!\left(\frac{1}{nB}\right)$
& $\displaystyle \frac{q_{1-\alpha}^2}{n^2}+O\!\left(\frac{1}{n^2B}\right)$
& $0$
& --- \\[1.5ex]
\textbf{Our work}
& $\displaystyle r^2+\frac{\log^2(Bn)d}{mn\mu^2}$
& $\displaystyle r^2+\frac{\log^2(Bn)}{mn\mu^2}$
& $\displaystyle r^2+\frac{\log^2(Bn)}{mn^2\mu^2}$
& $B$
& Yes \\[1.5ex]
\textbf{\cite{chadha2024resampling}}
& $\displaystyle r^2+\tilde{O}\bb{\frac{\sqrt{d}}{\varepsilon n^{3/2}}+\frac{d}{\varepsilon^2 n^2}}$
& Empirical
& $\times$
& $B\sqrt{n}$
& Yes \\[1.25ex]

\textbf{\cite{Dette2025GaussianDP}}
& $\displaystyle r^2+\frac{Bmd}{n^3\mu^2}$
& $\times$
& $\times$
& $B$
& No \\
\bottomrule
\end{tabular}
\caption{Comparison of the mean squared radius of the confidence region across estimator classes, and computational considerations for WR-sampling. We note that~\cite{shoham2025black}'s algorithm is similar to that of~\cite{Dette2025GaussianDP} and so we exclude it for brevity. For simplicity, for mean-like statistics we assume the global sensitivity to be $\Theta(\ell/n)$ for $\ell = \sqrt{d}$. The results of~\cite{chadha2024resampling} for quantile estimation are empirical; Moreover, their paper assumes the resamples to be private already.}
\label{tab:rw_comparison}
\end{table}





\subsection{Comparison with existing works}
In Table~\ref{tab:rw_comparison}, we compare our method to existing work.
When the estimator admits an asymptotically normal distribution, bootstrap and BLB both consistently estimate its sampling distribution~\citep{hall1992bootstrap,kleiner2014scalable}. \textit{Interior} sample quantiles (i.e., for quantiles between, say, the $0.05$ and $0.95$ quantiles) are asymptotically normal under mild regularity conditions \citep[e.g.][]{Vaart_1998}. 

Likewise, the method of \cite{chadha2024resampling} to quantify uncertainty privately applies to a large family of mean-like estimators and quantiles. Since BLB-Quant privatizes each resample, the length of the confidence interval depends on the privacy budget. Also, while this method achieves the lower bound for Gaussian means~\citep{Karwa2017FiniteSD}, it can be computationally expensive, requiring $O(B\sqrt{n})$ resample evaluations (cf.\ Table~\ref{tab:rw_comparison}).

\cite{Dette2025GaussianDP} analyze private $m$-out-of-$n$ resampling for estimators with global sensitivity of the form $O(\ell/n)$. Let \(q_{1-\alpha}\) be the target non-private resampling quantile of the pivot
\(\tau_n d(\widehat\theta_n,\theta)\), so the leading squared confidence radius (non-private) is
\(4q_{1-\alpha}^2/\tau_n^2\). 

Ignoring truncation at zero, with some regularity conditions on the resampling distribution (see Lemma~\ref{lem:quantile_finiteB} in the Appendix), the non-private confidence radius $r(\alpha)$ satisfies
\[
\mathbb E[r(\alpha)^2]
=
\frac{4q_{1-\alpha}^2}{\tau_n^2}
+
O\left(\frac{1}{B\tau_n^2}\right).
\]

\textbf{Our method.}
Algorithm~\ref{alg:gdp_quantile_subsample} forms the empirical resampling quantile
\(Q_{n,m}\), privatizes it once,
\(
\widetilde Q_{n,m}
=
\max\{Q_{n,m}+\sigma_m Z,0\}
\) for $
Z\sim N(0,1)$
and reports
\(
\widetilde\theta_n
\pm
\frac{\widetilde Q_{n,m}}{\tau_n}.
\)
For comparison, we isolate the contribution of the noisy quantile. For mean-like statistics with global sensitivity
\(\Delta_m\) on samples of size \(m\),
\(
\sigma_m
=
\frac{\tau_m\kappa_\star \Delta_m}{\mu},
\)
up to universal constants using Theorem~\ref{thm:global_guarantee}. Therefore
\[
\mathbb E_Z\!\left[\mathrm{Len}_{\mathrm{ours}}(\alpha)^2\right]
\approx
\frac{4q_{1-\alpha}^2}{\tau_n^2}
+O\bb{\frac{1}{B\tau_n^2}}+
\frac{4\tau_m^2\kappa_\star^2\Delta_m^2}
{\tau_n^2\mu^2}.
\]

Since \(\kappa_\star=O(\log(Bn))\) under with-replacement sampling and
\(\kappa_\star=1\) under sampling without replacement, our procedure replaces
the \(\sqrt B\)-type accumulation from repeatedly privatizing resample
statistics by a logarithmic dependence on \(B\) in the squared-length penalty.
In the $\sqrt{n}$-regime, where \(\tau_n=\sqrt n\), \(\tau_m=\sqrt m\), and
\(\Delta_m=\ell/m\), the privacy term becomes
\[
\Delta_{\mathrm{ours}}
\approx
\frac{4\kappa_\star^2\ell^2}{mn\mu^2}.
\]
Thus, for \(m\) close to \(n\) our approach pays only through \(\kappa_\star^2\), whereas the approaches of~\cite{Dette2025GaussianDP} and~\cite{shoham2025black} pay a factor
linear in \(B\).
Appendix section~\ref{sec:appendix:compare} has more details on the entries of Table~\ref{tab:rw_comparison}, and Appendices~\ref{app:quantiles} as well as~\ref{app:degenerate} explain the table entries for confidence lengths for quantiles and degenerate U-statistics.

\begin{table*}[t]
\centering
\label{tab:resampling_conditions}
\small
\setlength{\tabcolsep}{4pt}
\renewcommand{\arraystretch}{1.2}
\begin{tabularx}{\textwidth}{@{}L{0.13\textwidth} L{0.16\textwidth} L{0.19\textwidth} L{0.23\textwidth} Y@{}}
\toprule
\textbf{Category} & \textbf{Condition} & \textbf{Role} & \textbf{If the condition fails} & \textbf{Practical recommendation} \\
\midrule
\textbf{Classical resampling}
& Classical resampling validity
& Makes $Q_{n,m}$ approximate the target quantile
& The non-private CI may be miscalibrated
& Use estimator-specific resampling theory \\
\midrule
& Choice of $m$
& Controls the quality of the resampling approximation
& If $m$ is too small, subsampled estimators are poor in quality; if $m$ is too close to $n$, classical resampling validity fails
& To start, use $m \approx n^{\eta}$, e.g.\ $\eta \in [0.7,0.8]$ (this works well for $m$-out-of-$n$ and subsampling in~\cite{kleiner2014scalable}) \\
\midrule
& Choice of $B$
& Larger $B$ reduces Monte Carlo error
& If $B$ is small, the estimated CIs may be more variable
& Increase $B$ subject to computational constraints \\
\midrule
& Density bounded below near the target quantile (e.g.\ median)
& Controls nearby order-statistic spacings and hence the smooth sensitivity
& Private noise may be large
& Low-density regions and tail quantiles are more challenging \\
\midrule
\textbf{Privacy}
& Smooth sensitivity scales as $o(1/\tau_n)$
& Makes the private noise negligible
& Still private, but nominal coverage not guaranteed
& Use global/smooth sensitivity bounds as applicable \\
\midrule
\textbf{Applicability}
& Tractable and small global or smooth sensitivity
& Enables one-shot privatization of the final resampling quantile
& Our framework may not apply directly
& Use a problem-specific DP method (e.g.\ DP-SGD) \\
\bottomrule
\end{tabularx}
\caption{Conditions and practical recommendations for private resampling confidence intervals.}
\end{table*}

%% file: related.tex
\section{Related Work}
\label{sec: Related Works} 

\par Resampling poses additional challenges under differential privacy, since privatizing statistics computed on many resamples leads to privacy losses that accumulate under composition.~\cite{smith_2011} shows that for a broad class of asymptotically normal estimators, one can construct differentially private estimators that converge to the same asymptotic distribution as their non-private counterparts. As a result, differential privacy can be achieved with no asymptotic loss in statistical efficiency, including for parametric estimation and regression problems.  Several methods have been proposed in ~\cite{vu_slavkovic_2009,williams_mcsherry_2010,karwa_slavkovic_2012,karwa_et_al_2014,karwa_et_al_2015,karwa_slavkovic_2016,kazan2023test,wang_et_al_2015,kr} to account for the combined uncertainty due to sampling variability and the additional noise introduced for privacy in differentially private algorithms. However, the utility of these approaches—measured, for example, by the expected length of confidence intervals or the power of hypothesis tests—is generally evaluated empirically, or their conservativeness is established only under restrictive asymptotic assumptions.  Finite-sample guarantees for conservative differentially private inference have been established in a limited number of settings. In particular, ~\cite{sheffet_2017} studies confidence intervals for ordinary least-squares regression coefficients, which may be interpreted as a multivariate Gaussian generalization of the problem considered here, and ~\cite{cb} consider hypothesis testing for discrete distributions. 

~\cite{bt} showed that, under zero-concentrated differential privacy, bootstrap-based estimation of standard errors can be performed without additional privacy cost beyond that of releasing the point estimator. This allows the construction of conservative confidence intervals via post-processing of bootstrap replicates.  ~\cite{Karwa2017FiniteSD} study the construction of differentially private confidence intervals for the mean of a normal distribution, providing explicit finite-sample coverage guarantees for both known and unknown variance cases. They also establish lower bounds on the expected length of any DP confidence interval, showing near-optimal privacy–utility tradeoffs without assuming bounded data.     ~\cite{ferrando_wang_sheldon_2020} fit a (DP) estimator once, then simulate synthetic datasets from the fitted parametric model, re-run the (private) estimator on simulated data, and form a CI from the bootstrap distribution.  They explicitly argue that it captures uncertainty from both sampling and the DP mechanism, and prove asymptotic consistency in broad settings (including a special adaptation for linear regression).   

%% file: Notation/_notation.tex
\section{Preliminaries}
\label{sec:notation}
\input{Notation/differential}
\input{Notation/subsampling}

%% file: Notation/differential.tex
\subsection{$f$-Differential Privacy}
\label{sec:notation_dp}
Given a positive integer $n$ and domain $\cX$, a dataset $D$ of size $n$ is a finite collection of $n$ elements from $\cX$. Two datasets are \textit{neighbors} if they differ in exactly one element. Throughout this paper, we consider randomized algorithms whose input is a dataset $D \in \cX^n$ and output lies in a measurable space $\cY$.  A mechanism (or algorithm) is a randomized function $\mathcal{A} : \cX^n \to \cY.$ Formally, such an algorithm induces, for each dataset $D$, an output distribution $\mathcal{A}(D)$ on $\cY$.

The notion of differential privacy we consider is $f$-differential privacy~\citep{dong2022gaussian}, which is based on hypothesis testing between pairs of output distributions. 

Suppose $P$ and $Q$ are two probability distributions on $\cY$. Given an observation $Y$ drawn from either $P$ or $Q$, consider the hypothesis testing problem with null and alternative hypotheses
\[
H_0: Y \sim P
\;\; ; \;\;
H_1: Y \sim Q,
\]
respectively. A \textbf{test} is any (possibly randomized) function $\phi:\cY\to \{0,1 \}$. The \textbf{type I error} $\alpha_\phi := \E_{Y\sim P}[\phi(Y)]$ of a test $\phi$ is the probability that the test outputs $1$ when $Y\sim P$, thereby incorrectly concluding that $Y$ was drawn from $Q$, and the \textbf{type II error} $\beta_\phi := \E_{Y\sim Q}[1-\phi(Y)]$ of the test $\phi$ is the probability that the test outputs $0$ when $Y\sim Q$, thereby incorrectly concluding that $Y$ was drawn from $P$.



\begin{definition}[Trade-Off function]
\label{def:tradeoff}
Given two probability distributions $P$ and $Q$ on the same measurable space $\cX$, the \emph{trade-off function} between $P$ and $Q$ is the function $T(P,Q):[0,1] \to [0,1]$ defined as
\[T(P, Q)(\alpha) = \inf_{\phi : \alpha_{\phi} \le \alpha} \beta_\phi. \]

A function $f:[0,1] \to [0,1]$ is called a \textbf{trade-off function} if $f = T(P,Q)$ for some distributions $P$ and $Q$ on the same measurable space. 
\end{definition}

By an abuse of notation, for random variables $X$ and $Y$, we write $T(X,Y)$ for the trade-off function between their distributions.
\begin{proposition}[\cite{dong2022gaussian}]\label{prop:tradeoff_characterization}
A function $f:[0,1] \to [0,1]$ is a trade-off function (between some two distributions $P$ and $Q$) if and only if $f$ is convex, continuous, and non-increasing in $[0,1]$, and satisfies $f(\alpha) \le 1-\alpha$ for $\alpha \in [0,1]$.
\end{proposition}

Proposition~\ref{prop:tradeoff_characterization} characterizes the functions that can arise as trade-off functions, and hence the functions that can serve as benchmarks for indistinguishability. To apply this hypothesis-testing framework to privacy, consider a randomized algorithm $\mathcal{A}$ whose input is a dataset. Each dataset $D$ induces a distribution over outputs, represented by the random variable $\mathcal{A}(D)$. For neighboring datasets $D$ and $D'$, the trade-off function $T(\mathcal{A}(D),\mathcal{A}(D'))$ describes how well an observer can distinguish the corresponding output distributions.

\begin{definition}[$f$-differential privacy]
Let $f:[0,1] \to [0,1]$ be a trade-off function. An algorithm $\mathcal{A}$ is $f$-\emph{DP} if for all neighboring datasets $D$ and $D'$ and all $\alpha \in [0,1]$,
\[
T(\mathcal{A}(D), \mathcal{A}(D'))(\alpha) \ge f(\alpha).
\]
\end{definition}

In other words, if $f = T(P,Q)$ for distributions $P$ and $Q$, then distinguishing between (the outputs of) any two neighboring datasets is at least as difficult as distinguishing between the distributions $P$ and $Q$ based on a single sample. 

$f$-differential privacy captures the standard notion of central privacy~\cite{dwork2006differential, dong2022gaussian}.

\begin{definition}[$(\eps, \delta)$-DP]
For any $\eps \ge 0$ and $\delta \in [0,1]$, algorithm $\mathcal{A}$ is $(\eps, \delta)$-DP if and only if it is $f_{\eps, \delta}$-DP, where
\bas{
f_{\eps, \delta}(\alpha) := \max \{1-\delta-e^{\eps}\alpha, e^{-\eps}(1-\delta-\alpha), 0\} \;\;\;\; \forall \alpha \in [0,1].
}
\end{definition}



\subsubsection{Gaussian Differential Privacy}
\label{subsec:gdp}

Gaussian differential privacy~\cite{dong2022gaussian} (GDP) measures privacy by comparing the distinguishability of the outputs on neighboring datasets to a canonical
Gaussian hypothesis-testing problem. Let $\mu \ge 0$, and consider the problem of distinguishing between the distributions
\[
    P = \mathcal N(0,1)
    \qquad\text{and}\qquad
    Q = \mathcal N(\mu,1).
\]
The corresponding trade-off function $G_\mu := T\bigl(\mathcal N(0,1),\mathcal N(\mu,1)\bigr)$ has the explicit form
\[
    G_\mu(\alpha)
    =
    \Phi\!\left(\Phi^{-1}(1-\alpha)-\mu\right)
\]
for all $\alpha \in [0,1]$, where $\Phi$ denotes the CDF of the standard normal distribution.\footnote{For each $\alpha \in [0,1]$, the test $\phi_{\alpha}$ that rejects the null if and only if the observed sample exceeds $\Phi^{-1}(1-\alpha)$ has type I error $\alpha$ and achieves the minimum type II error $\Phi(\Phi^{-1}(1-\alpha)-\mu)$.}

An algorithm $\mathcal{A}$ is said to be $\mu$-GDP if distinguishing the outputs of $\mathcal{A}$ on
neighboring datasets is no easier than distinguishing
$\mathcal N(0,1)$ from $\mathcal N(\mu,1)$ using a single sample.

\begin{definition}[Gaussian Differential Privacy]
\label{def:pure_gdp}
For any $\mu \ge 0$, an algorithm $\mathcal{A}$ is said to be
$\mu$-Gaussian differentially private ($\mu$-GDP) if it is $G_{\mu}$-GDP. In other words, for every pair of
neighboring datasets $D$ and $D'$ and all $\alpha \in [0,1]$,
\[
    T\bigl(\mathcal A(D),\mathcal A(D')\bigr)(\alpha)
    \geq \Phi(\Phi^{-1}(1-\alpha)-\mu).
\]
\end{definition}

\textbf{Privacy composition.} Gaussian differential privacy admits an exact composition rule. Let
\[
Y_i = \mathcal{A}_i(D;Y_1,\ldots,Y_{i-1}), \qquad i=1,\ldots,k,
\]
be the outputs of $k$ mechanisms $\mathcal{A}_1, \dots, \mathcal{A}_k$ and suppose that, for every fixed realization $y_{1:(i-1)}$ of the previous outputs, the conditional mechanism $D \mapsto \mathcal{A}_i(D;y_1,\ldots,y_{i-1})$
is $\mu_i$-GDP. Then, the mechanism $\mathcal{A}$ that releases the joint transcript $Y_{1:k}=(Y_1,\ldots,Y_k)$ is $\mu$-GDP with
\ba{
\mu=\sqrt{\mu_1^2+\cdots+\mu_k^2}\,. \label{eq:pure_gdp_composition}
}

\begin{proposition}[Equivalence between GDP and DP]
For any $\mu > 0$, an algorithm $\mathcal{A}$ is $\mu$-GDP if and only if it is $(\eps, \delta_{\mu}(\eps))$-DP for all $\eps \ge 0$, where \[\delta_{\mu}(\eps) = \Phi\bb{-\frac{\eps}{\mu} + \frac{\mu}{2}} - e^{\eps} \Phi\bb{-\frac{\eps}{\mu} - \frac{\mu}{2}}.\]
\end{proposition}

\begin{proposition}[Gaussian mechanism]
\label{prop:gaussian_mechanism}
Suppose $\mathcal{A} : \cX^n \to \mathbb{R}^d$ has $\ell_2$-global sensitivity at most $\Delta$, that is, $\norm{\mathcal{A}(D) - \mathcal{A}(D')}_2 \le \Delta$ for all neighboring datasets $D$ and $D'$. Then, for any $\mu > 0$, the mechanism $\tilde{\mathcal{A}}$ that samples $Z \sim \mathcal{N}\bb{0, \id_d}$ and outputs
\[\tilde{\mathcal{A}}(S) = \mathcal{A}(S) + \frac{\Delta}{\mu} \cdot Z\]
is $\mu$-GDP.
\end{proposition}

%% file: Notation/subsampling.tex
\subsection{Assumptions on subsampling and $m$-out-of-$n$ sampling}
For all our utility results, we will need the following standard assumptions from~\cite{RESAMPLING,politis1999subsampling}. Let $\tau_n$ be a scaling factor (typically, it is equal to $\sqrt{n}$). Let $F_n(t):=P(\tau_n d(\hatn,\theta)\leq t)$ and $F(t) := \lim_{n \to \infty} F_n(t)$ denote the CDF of the empirical distribution and the limiting distribution respectively. For $x,y\in \mathbb{R}^d$, we use $d(x,y)=\|x-y\|_2$.
\begin{assumption}\label{ass:limit}
    We assume that 
    \begin{enumerate}
        \item $F$ is non-degenerate and continuous.
        \item $\sup_t |F_n(t)-F(t)|\rightarrow 0.$ 
        \item $\tau_m/\tau_n\rightarrow 0$, $m/n\rightarrow 0$, and $B\rightarrow\infty$.
        \item (Resampling validity.) $\Pr(\tau_n d(\hatn,\theta)\leq \qnm)=1-\alpha+o_n(1)$, where $\qnm$ denotes the scaled quantile defined in Algorithm~\ref{alg:gdp_quantile_subsample}. In particular, $Q_{n,m} \stackrel{p}{\rightarrow} F^{-1}(1-\alpha).$

    \end{enumerate}
\end{assumption}
These assumptions can be verified using straightforward probabilistic arguments for both $m$-out-of-$n$ sampling and subsampling. For completeness, we provide a detailed discussion in Appendix~\ref{sec:resampDetails}.


%% file: Algorithms/_algorithms.tex
\input{Algorithms/subsampling_methods}
\input{Algorithms/global}
\input{Algorithms/smoothed}
\input{Algorithms/applications}

%% file: Algorithms/subsampling_methods.tex
\section{Private confidence regions for mean-like estimators}

In this section, we develop a one-shot procedure for constructing differentially private confidence regions for the mean-like estimators introduced in Definition~\ref{def:gs-bounded}. Our construction applies to both sampling without replacement and $m$-out-of-$n$ sampling with replacement. As in standard resampling-based inference, we compute the estimator on the full dataset and on $B$ resamples of size $m$, and use an empirical quantile of the resulting resampling errors to determine the radius of the confidence region. The main difference is that we do not privatize the individual resample estimates. Instead, we add noise only to the final resampling quantile, together with a separate perturbation of the full-sample
estimator.

A technical issue that arises is controlling the sensitivity of this empirical quantile. Under sampling without replacement, a changed observation can appear at most once in each resample. Under sampling with replacement, the same observation may occur multiple times, and the sensitivity is governed by the maximum multiplicity of any original observation within the resamples. We first give several examples of mean-like estimators. We then establish a high-probability bound on the maximum multiplicity under with-replacement sampling and introduce a multiplicity-conditioned sampling procedure. Finally, we present the resulting confidence-region algorithm $\dpgs$ and establish its privacy and utility guarantees.


Several standard classes of estimators satisfy the sensitivity condition $\Delta_2(\theta_n) \leq \ell_n/n$ in Definition~1. We state three representative examples.

\begin{enumerate}

    \item[(i)] \emph{Bounded averages and their Lipschitz transforms.} Suppose
    \[
    \theta_n(D)=\frac{1}{n}\sum_{i=1}^n \psi_n(X_i)\in\mathbb{R}^d \quad\text{and}\quad \|\psi_n(x)\|_2\le R_{n,d}\ \ \forall x,
    \]
    where $D$ is the dataset $\{X_1, X_2, \dots, X_n\}$. Then, for any neighboring datasets $D$ and $D'$, ${\lVert \theta_n(D)-\theta_n(D')\rVert}_2 \le {2 R_{n,d}}/{n}.$ It follows that
    \[
    \Delta_2(\theta_n) \le \frac{2R_{n,d}}{n}.
    \]
    For instance, if $\psi$ is the identity function (capturing the mean-estimation problem) and $\|X_i\|_\infty\leq c_n$ for all $i$, then $R_{n,d}= 2 c_n \sqrt{d}$. This can also be applied to light-tailed distributions by first clipping the data: for instance, if it is known that the $X_i$'s are sub-Gaussian with variance proxy $\sigma_{\rm {proxy}}^2$, we can clip each coordinate at $\sigma_{\rm {proxy}} \sqrt{\log n}$ to get $R_{n,d}=\Theta(\sigma_{\rm {proxy}}^2\sqrt{d\log n})$.
    
    $\ell_2$-Lipschitz functions of averages also have global sensitivity of the same form with the sensitivity multiplied by the Lipschitz constant; if $g$ is $L$-Lipschitz in the $\ell_2$ norm, then
    \[
        \Delta_2(g\circ\theta_n)
        \leq
        \frac{2L R_{n,d}}{n}.
    \]

    \item[(ii)] \emph{Strongly convex regularized $M$-estimators.} Let $\ell(\cdot;x)$ be a convex loss function which is $G$-Lipschitz in $y$ (in the $\ell_2$ norm). Then, the $M$-estimator
    \[
        \theta(D)
        =
        \arg\min_{y\in\mathbb{R}^d}
        \left\{
            \frac{1}{n}\sum_{i=1}^n \ell(y;X_i)
            +
            \frac{\lambda}{2}\|y\|_2^2
        \right\},
    \]
    has global sensitivity $\Delta_2(\theta(D)) \le {2G}/{\lambda n}$ (see Lemma~7 in~\cite{Chaudhuri2011DPERM}).

    \item[(iii)] \emph{$U$-statistics with a bounded kernel.}
    Let
    \[
        \theta_n(D)
        =
        \binom{n}{r}^{-1}
        \sum_{1\leq i_1<\cdots<i_r\leq n}
        h_n(X_{i_1},\ldots,X_{i_r})
    \]
    be a $U$-statistic of order $r$, where $r$ may depend on $n$.
    Suppose the range of $h_n$ has $\ell_2$-diameter at most $b_n$, that is,
    $\sup_{z,z'\in\mathcal{X}^r} \|h_n(z)-h_n(z')\|_2 \leq b_n.$
    
    Note that changing one observation affects at most
    $\binom{n-1}{r-1}$ of the summands. Therefore,
    \[
        \Delta_2(\theta_n)
        \leq {\binom{n-1}{r-1}}{\binom{n}{r}}^{-1} b_n =
        \frac{r b_n}{n}.
    \]
    Thus, $\theta_n$ is mean-like with $\ell_n = r b_n$.
\end{enumerate}


\subsection{Bounding the multiplicity for sampling with replacement}
\label{sec:bounding_multiplicity}
Under with-replacement sampling, an index $j\in[n]$ may appear multiple times in a subsample $S$. Even if $D$ and $D'$ differ only on the data point $X_j$, the resampled datasets $D[S] = \{X_i\}_{i \in S}$ and $D'[S]$ can differ in as many as $m$ entries; if $\kappa(S)$ is the maximum multiplicity of any data point in $S$ and $\Delta_m$ is the global sensitivity of $\theta$ on datasets of size $m$, then the worst-case change on one subsample can be $\Delta_m\,\kappa(S)$, which is large in the worst case.

To control the sensitivity uniformly over $B$ subsamples $S_1,\dots,S_B$, we work on the event that all $\kappa(S_i)\leq \kappa_*$, for $i\in [B]$.
Under this event, $D[S_i]$ and $D'[S_i]$ differ in at most $\kstar$ data points.

\begin{proposition}[Multiplicity bound for sampling with replacement]
\label{prop:kstar_choice}
Let $m<n$, $B\in\mathbb Z_+$, and $\rho\in(0,1)$. For each $i \in [B]$, suppose we draw $S_i = \{S_{i,t}\}_{t \in [m]},$ where $S_{i,t}$ are $\iid$ draws from $[n]$, independently across $i$. Let
\[
    \kappa(S_i)
    :=
    \max_{j\in[n]}
    \#\{t\in[m]:S_i(t)=j\},
\]
and $\mathcal{E}_{\kappa} :=\cap_{i=1}^B\{\kappa(S_i)\leq\kappa\}.$ Then, the constant
\ba{
    \kappa_\star(n,m,B,\rho)
    :=
    \min\left\{
        m,\,
        \left\lceil
            \log \frac{Bn}{\rho} \cdot \bb{W_0\bb{\frac{n}{me} \log \frac{Bn}{\rho}}}^{-1}
        \right\rceil-1
    \right\} \label{def:kstar}
}
satisfies $\Pr(\mathcal E_{\kappa_\star}^{\,c})\leq\rho.$\footnote{The function $W_0$ denotes the principal branch of the Lambert $W$ function. For any positive real number $x$, $W_0(x)$ is the unique positive real number such that $W_0(x)e^{W_0(x)} = x.$}
\end{proposition}

To get an idea for the size of $\kstar$ in Proposition~\ref{prop:kstar_choice}, let $A := \log (Bn/\rho)$ and $M := \log (An/m).$ Then, $W_0(An/me) = M-1-\log(M-1) + o(1)$. For $M \ge 2,$ this implies
\[
\kstar = \Theta\bb{\frac{\log (Bn/\rho)}{\log (n/m) + \log \log (Bn/\rho)}},
\]
which is \textit{sub-logarithmic} in both $B$ and $n$.

The proof of Proposition~\ref{prop:kstar_choice} is deferred to Appendix~\ref{sec:resampDetails}. Algorithm~\ref{alg:subsample_mult} generates $B$ $m$-out-of-$n$ with-replacement indices $S_1, \dots, S_B$ and returns failure iff some $j \in [n]$ is sampled more than $\kappa$ times in some $S_i$. By Proposition~\ref{prop:kstar_choice}, the algorithm fails with probability at most $\rho$ if $\kappa \ge \kappa_\star(n,m,B,\rho)$.

\begin{algorithm}[H]
\caption{\texttt{\textbf{M}ultiplicity-\textbf{C}onditioned $m$-out-of-$n$ \textbf{S}ampling}}
\label{alg:subsample_mult}
\KwIn{Sample size $n$, resample size $m$, number of resamples $B$, threshold $\kappa$}


\For{$i \gets 1$ \KwTo $B$}{

   Draw the multi-set $S_i=\{S_{i1},\ldots,S_{im}\}$, where $S_{i1},\ldots,S_{im}\stackrel{\mathrm{iid}}{\sim}\mathrm{Unif}([n])$
    
    $\kappa_i \gets \max_{j \in [n]} \bigl|\{t \in [m] : S_{it} = j\}\bigr|$
    
    \If{$\kappa_i > \kappa$}{
        \Return $\bot$\;
    }
}
\Return $(S_1,\dots,S_B)$\;
\end{algorithm}

%% file: Algorithms/global.tex
\subsection{Algorithm for mean-like estimators}
\label{sec:global_algorithm}

We now present \dpgs, a procedure for constructing differentially private confidence sets using subsampling-based quantile estimation when the statistic $\theta$ has bounded global $\ell_2$-sensitivity.

Let $\widehat{\theta}_n=\theta(D)$ be the estimator computed on the full dataset, and let $\Delta_n$ and $\Delta_m$ denote the
global sensitivities of $\theta$ on datasets of sizes $n$ and $m$,
respectively. $\dpgs$ proceeds as follows:

\begin{enumerate}
    \item
    Draw $B$ resamples $S_1,\ldots,S_B$ of size $m$ each. For sampling
    without replacement, set $\kstar=1$. For sampling with
    replacement, set $\kstar = \kstar(n,m,B,\rho)$ (equation~\eqref{def:kstar}) and accept the resamples only when the multiplicity of every observation in every resample is at most $\kstar$.

    \item
    Compute the resample estimates $\theta_i^\star=\theta(D[S_i])$ for $i = 1,2,\dots, B.$ These estimates are used only in the algorithm and are \textit{not released}.

    \item
    Compute the $\tau_m$-scaled empirical resampling quantile
    \[
        Q_{n,m}
        =
        \tau_m \cdot
        \texttt{Quantile}_{1-\alpha}
        \bb{
            \norms{\theta_1^\star-\widehat{\theta}_n}_2,
            \norms{\theta_2^\star-\widehat{\theta}_n}_2,
            \ldots,
            \norms{\theta_B^\star-\widehat{\theta}_n}_2
        }.
    \]

    \item
    Privatize the center and the radius by releasing
    \[
        \widetilde{\theta}_n
        =
        \widehat{\theta}_n+\nu_n Z_2,
        \qquad
        \widetilde{Q}_{n,m}
        =
        \max\{Q_{n,m}+\sigma_m Z_1,0\},
    \]
    where $Z_1\sim N(0,1)$ and $Z_2\sim N(0,\id_d)$ are independent, and the
    noise scales $\nu_n$ and $\sigma_m$ are calibrated to the sensitivities
    of the full-sample estimator and the resampling quantile, respectively.

    \item
    Return the confidence region
    \[
        \mathcal{C}_{n,\alpha}
        =
        \left\{
            u\in\mathbb{R}^d:
            \norms{u-\widetilde{\theta}_n}_2
            <
            \frac{\widetilde{Q}_{n,m}}{\tau_n}
        \right\}.
    \]
\end{enumerate}

The sensitivity of each resampling distance
$\norms{\theta_i^\star-\widehat{\theta}_n}_2$ is at most
$\kstar\Delta_m+\Delta_n$. Consequently, the empirical quantile has
the same sensitivity bound, since any order statistic is $1$-Lipschitz with
respect to the maximum change in its inputs. Thus, up to the allocation of
the privacy budget between the center and the radius, the required noise
scale for the quantile is proportional to $\tau_m\bigl(\kstar\Delta_m+\Delta_n\bigr).$

\begin{algorithm}[ht]
\caption{\label{alg:gdp_quantile_subsample}\texttt{ Confidence Regions via Global Sensitivity (\dpgs)}
}
\KwIn{Dataset $D$, GDP parameter $\mu$, functional $\theta(\cdot)$, confidence parameter $\alpha$, global sensitivity $\Delta(\cdot)$, scale function $\tau(\cdot)$, number of subsamples $B$, subsample size $m$, failure probability $\rho$, subsampling mode $\textsf{Mode} \in \{\WR, \WOR\}$)}


\noindent
\begin{minipage}[t]{0.49\linewidth}
\vspace{0pt}
\If{\textsf{Mode} = \WR}{
    $\kstar \gets \kappa_\star(n,m,B,\rho)$
    of equation~\eqref{def:kstar}

    $(S_1,S_2,\dots,S_B)
    \gets \texttt{MCS}(n,m,B,\kstar)$

    \If{$(S_1,S_2,\dots,S_B)=\bot$}{
        \Return $\bot$
    }
}
\end{minipage}
\hfill
\begin{minipage}[t]{0.49\linewidth}
\vspace{0pt}
\If{\textsf{Mode} = \WOR}{
    $\kstar \gets 1$

    $(S_1,S_2,\dots,S_B)
    \gets \texttt{Subsample}(n,m,B)$
}
\end{minipage}

\vspace{0.75ex}


$\hatn \gets \theta(D)$ \Comment{\textit{Estimator on full dataset}}


\For{$i \gets 1$ \KwTo $B$}{
  $\theta^*_i \gets \theta(D[S_i])$ \Comment{\textit{Resample estimates}}
}

$\qnm \gets \tau_m
        \texttt{Quantile}_{1-\alpha}
        \bb{
            \norms{\theta_1^\star-\widehat{\theta}_n}_2,
            \ldots,
            \norms{\theta_B^\star-\widehat{\theta}_n}_2
        }$ \Comment{\textit{Empirical quantile (scaled)}}

\vspace{0.75ex}


$Z_1 \sim \mathcal{N}(0, 1) ; Z_2 \sim \mathcal{N}(0, \id_d)$

$\tiln \gets \hatn +\nu_n Z_2,$ where $\nu_n = \frac{\sqrt{2} \Delta_n}{\mu}$

$\qnmp \gets \max(\qnm + \sigma_m Z_1,0),$ where $\sigma_m = \dfrac{\sqrt{2} \tau_m (\kstar \Delta_m + \Delta_n)}{\mu}$

\Return $(\tiln,\qnmp)$
\end{algorithm}

\begin{theorem}[Privacy and utility of \dpgs]
\label{thm:global_guarantee}
Algorithm~\ref{alg:gdp_quantile_subsample} is $\mu$-GDP.  The output $(\tiln, \qnmp)$ of Algorithm~\ref{alg:gdp_quantile_subsample} satisfies
\[
\Pr\bb{\norms{\theta - \tiln}_2 < \frac{1}{\tau_n}\qnmp} = 1-\alpha + o_n(1),
\]
as long as assumptions~\ref{ass:limit} hold, the failure probability $\rho = o_n(1)$, and $\gamma = o_n(1)$, where $\gamma := \sigma_{m}\sqrt{2\log (1/\rho)} +  \tau_n\nu_n(\sqrt{d}+\sqrt{2\log (1/\rho)}).$ 
Here, $\sigma_m=\sqrt{2}\tau_m\mu^{-1} (\kstar\Delta_m+\Delta_n)$ and $\nu_n = \sqrt{2}\mu^{-1}\Delta_n.$

The confidence radius has a private overhead with scale $\tau_n^{-1} \sigma_m = \Theta(\tau_m \tau_n^{-1}\mu^{-1} (\kstar\Delta_m+\Delta_n))$.
\end{theorem}

The proof of Theorem~\ref{thm:global_guarantee} can be found in Appendix~\ref{app:global}.



\rd 
\bk

%% file: Algorithms/smoothed.tex
\section{Private confidence regions for S-friendly estimators}
\label{sec:smooth_sensitivity}

\subsection{Approximate Gaussian Differential Privacy}
\label{sec:approximate_GDP}
$\mu$-GDP requires the output distributions on neighboring datasets
to be at least as difficult to distinguish as the Gaussian pair
$\mathcal N(0,1)$ and $\mathcal N(\mu,1)$. We consider an approximate
version in which one of the output distributions may first be perturbed by
a small amount $\delta$ in total variation.

For $\mu\geq 0$ and $\delta\in[0,1)$, define the trade-off
function
\ba{
    G_{\mu,\delta}(\alpha)
    :=
    \max\{G_\mu(\alpha)-\delta,0\},
    \qquad \alpha\in[0,1]. \label{def:G_mu_delta}
}
The following proposition gives an interpretable, equivalent definition of this trade-off function.

\begin{proposition}[Equivalent characterizations of approximate GDP]
\label{prop:approx_gdp_equiv}
Let $P$ and $Q$ be probability distributions on the same measurable space.
Then $T(P,Q)\geq \widetilde G_{\mu,\delta}$
if and only if there exists a probability distribution $R$ such that
\[
    T(P,R)\geq G_\mu
    \qquad\text{and}\qquad
    \operatorname{TV}(Q,R)\leq\delta.
\]



\end{proposition}

\begin{definition}[Approximate Gaussian differential privacy]\label{def:approx_gdp}
A mechanism $\mathcal M$ is said to satisfy
$(\mu,\delta)$-\emph{Gaussian differential privacy}
($( \mu,\delta)$-GDP) if, for every pair of neighboring datasets
$D$ and $D'$,
\[
    T\bigl(\mathcal M(D),\mathcal M(D')\bigr)(\alpha)
    \geq
    G_{\mu,\delta}(\alpha)
    \qquad
    \text{for all }\alpha\in[0,1].
\]
\end{definition}

Since $G_{\mu,0}=G_\mu,$ $(\mu,0)$-GDP coincides with ordinary $\mu$-GDP. The class of $(\mu,\delta)$-GDP mechanisms is closed under composition.

\begin{proposition}[Adaptive composition]
\label{prop:approx_gdp_composition}
Let $\mathcal{A}_i : \mathcal{X} \times \mathcal{Y}_1 \times \dots \times \mathcal{Y}_{i-1} \to \mathcal{Y}_{i}$ for $i \in [k]$ be mechanisms such that
$\mathcal{A}_i(\cdot, y_1, y_2, \dots, y_{i-1})$ is $(\mu_i, \delta_i)$-GDP for every $y_1 \in Y_1, \dots, y_{i-1} \in Y_{i-1}$ for all $i \in [k]$. Then, the $k$-fold adaptive composition $\mathcal{A} : X \to Y_1 \times \dots \times Y_k$ that jointly releases the outputs of all $k$ mechanisms is $(\mu, \delta)$-GDP, where
\[
\mu^2 = \mu_1^2 + \mu_2^2 + \dots + \mu_k^2 \,\, \text{ and } \,\, 1-\delta = (1-\delta_1)(1-\delta_2) \dots (1-\delta_k).
\]
\end{proposition}

The proofs of Propositions~\ref{prop:approx_gdp_equiv} and~\ref{prop:approx_gdp_composition} are deferred to Appendix~\ref{app:ptr_gdp_proof}.



\subsection{Smooth sensitivity and Propose-Test-Release (PTR)}
\label{sec:ptr_gdp}


\begin{algorithm}[H]
\caption{\texttt{PTR-GDP}}
\label{alg:ptr_gdp}
\KwIn{Dataset $D$, statistic $\theta(\cdot)$, smoothness parameter $\beta$, smooth sensitivity $\SS_{\beta}$, privacy parameters $\mu$ and $\delta$}

Compute the smooth sensitivity bound $\SS_{\beta}(D)$

Draw $Z_1 \sim \mathcal N(0,1), Z_2 \sim \mathcal N(0, \id_d)$

$\mathcal{B}(D) \gets \SS_{\beta}(D) \cdot \exp\bb{\frac{\beta}{\mu} \cdot Z_1 + \frac{\beta}{\mu} \cdot \Phi^{-1}(1-\delta)}, \;\;\; \tilde{\theta}(D) \gets \theta(D) + \frac{\mathcal{B}(D)}{\mu} \cdot Z_2$

\Return $\tilde{\theta}(D)$

    
\end{algorithm}

Let $f : \cX^n \to \R$ be a statistic with some local sensitivity bound $\LS_f$ (equation~\ref{eq:definition_local_sensitivity}). Recall that a function $\SS_{\beta}$ is said to be a $\beta$-smooth upper bound on if
\[
\LS_{f}(D) \le \SS_{\beta}(D)
\quad\text{and}\quad
\SS_{\beta}(D) \le e^{\beta}\, \SS_{\beta}(D') \quad \text{for all adjacent } D \sim D',
\]
In this section, we adapt the Propose-Test-Release (PTR) framework~\citep{dwork2009differential} for GDP,
leveraging the facts that smooth sensitivity bounds become Lipschitz after taking logarithms. The smooth sensitivity can be much smaller than the global sensitivity on typical datasets. While the smooth sensitivity may not be efficiently computable in general, there are many estimators such as the sample quantile or $U$-statistics where one can compute it efficiently~\citep{nissim2007smooth, fletcher2017differentially, chaudhuri2024differentially}.

Algorithm~\ref{alg:ptr_gdp} is a PTR-style algorithm that uses Gaussian noise and yields an \emph{approximate} $\mu$-GDP guarantee, as long a smooth upper bound $\SS_\beta$ for $\LS$ is computable. 
The PTR step above is an analog of the private sensitivity framework of \cite[Lemma~4.4]{Karwa2011graph}. We state the corresponding GDP version below.
\begin{lemma}[PTR-GDP is private]
\label{lem:ptr_gdp_privacy}
Suppose $\mu_1, \mu_2 > 0$ and let $\mathcal B : \mathcal{X}^n \to (0, \infty)$ be a $\mu_1$-GDP mechanism such that $\Pr\!\left[
        \mathcal B(D) \geq \LS(\theta, D)
    \right]
    \geq 1-\delta$
for every dataset $D$. Then, the output \[\bb{\mathcal{B}(D),~\theta(D) + \frac{\mathcal{B}(D)}{\mu_2} \cdot \mathcal{N}(0,\id_d)}\] is $(\sqrt{\mu_1^2+\mu_2^2},\delta)\text{-GDP}.$ In particular, Algorithm~\ref{alg:ptr_gdp} is $(\sqrt{2} \mu, \delta)\text{-GDP}.$
\end{lemma}

The proof of Lemma~\ref{lem:ptr_gdp_privacy} can be found in Appendix~\ref{app:ptr_gdp_proof}.

\subsection{Private confidence regions for S-friendly estimators in Definition~\ref{def:ss-bounded}}

\begin{algorithm}[ht]
\caption{\label{alg:smoothed_quantile_subsample}\texttt{Confidence Regions via \textbf{S}moothed \textbf{S}ensitivity (\dpss)}
}
 \KwIn{Dataset $D$, statistic $\theta(\cdot)$, confidence parameter $\alpha$, smoothness parameter $\beta$, smooth sensitivity $\SS_{\beta}$, privacy parameters $\mu$ and $\delta$, failure probability $\rho$, subsampling method $\textsc{Mode} \in \{\WR, \WoR\}$}

\noindent
\begin{minipage}[t]{0.49\linewidth}
\vspace{0pt}
\If{\textsf{Mode} = \WR}{
    $\kstar \gets \kappa_\star(n,m,B,\rho)$
    of equation~\eqref{def:kstar}

    $(S_1,S_2,\dots,S_B)
    \gets \texttt{MCS}(n,m,B,\kstar)$

    \If{$(S_1,S_2,\dots,S_B)=\bot$}{
        \Return $\bot$
    }
}
\end{minipage}
\hfill
\begin{minipage}[t]{0.49\linewidth}
\vspace{0pt}
\If{\textsf{Mode} = \WOR}{
    $\kstar \gets 1$

    $(S_1,S_2,\dots,S_B)
    \gets \texttt{Subsample}(n,m,B)$
}
\end{minipage}

\vspace{0.75ex}

$\hatn \gets \theta(D)$

$\tiln \gets \text{\texttt{PTR-GDP}}(D, \theta, \beta, \SS_{\beta}, \mu/2, \delta/2)$ 

\For{$i \gets 1$ \KwTo $B$}{
  $\theta^*_i \gets \theta(D[S_i])$
}

$\qnm \gets \tau_m
        \texttt{Quantile}_{1-\alpha}
        \bb{
            \norms{\theta_1^\star-\widehat{\theta}_n}_2,
            \ldots,
            \norms{\theta_B^\star-\widehat{\theta}_n}_2
        }$ 

Define the function
$s(D) := \max_{i \in [B]} \SS_{\beta}(D[S_i]) \cdot {(e^{\kstar \beta}-1)}/{(e^{\beta}-1)} + \SS_{\beta}(D)$

$\qnmp \gets \text{\texttt{PTR-GDP}}(D, \qnm, \kstar \beta, \tau_m \cdot s(D), \mu/2, \delta/2)$

$\qnmp \gets \max\{\qnmp, 0\}$

\Return $(\tiln, \qnmp)$
\end{algorithm}

Let $\theta: \mathcal{X}^n \to \mathbb{R}^d$ be an estimator and $\SS_{\beta}$ be a $\beta$-smooth upper bound on the local sensitivity of $\theta$.
Consider a (multi)-set of indices $S_i$ of size $m$ drawn uniformly from $[n]$. Let 
$s_i = \SS_{\beta}(\theta, D[S_i])$ 
be the smooth sensitivity for the dataset $D[S_i]$ (where the $m$ samples in $D[S_i]$ are viewed as distinct elements when computing the smooth sensitivity, regardless of the sampling method).

\begin{lemma}\label{lem:smooth_each_subsample}
Let \[s(D) := \max_{i \in [B]} \SS_{\beta}(\theta, D[S_i]) \cdot \frac{e^{\kstar \beta}-1}{e^{\beta}-1} + \SS_{\beta}(\theta, D).\]
Then, $\tau_m s(D)$ is a $\kstar \beta$-smooth upper bound on the local sensitivity of $\qnm$ of Algorithm~\ref{alg:smoothed_quantile_subsample}.
\end{lemma}

The proof of Lemma~\ref{lem:smooth_each_subsample} is in Appendix~\ref{app:alg4}.

\begin{theorem}\label{thm:smooth_utility}
\dpss (Algorithm~\ref{alg:smoothed_quantile_subsample}) is $(\mu, \delta)$-GDP. Suppose for any $N \in \mathbb{Z}_{+}$ and $\eta \in (0,1),$ the smooth sensitivity function $\SS_{\beta}$ satisfies $\SS_{\beta}(\theta, D) \le \Gamma_{N,\eta}$ with probability at least $1-\eta$ for datasets $D$ of size $N$. Then, with the choice of the smoothness parameter
\ba{
\beta := \frac{1}{2\kstar} \min\bb{1, \frac{\mu}{4 \sqrt{\log (B/(\delta \rho))}}}, \label{eq:beta_choice}
}
the output $(\tiln, \qnmp)$ of Algorithm~\ref{alg:smoothed_quantile_subsample} satisfies
\[
\Pr( d(\theta, \tiln) < \tau_n^{-1} \qnmp) = 1-\alpha + o_n(1)
\]
provided Assumption~\ref{ass:limit} holds, the failure probability $\rho = o_n(1)$, and 
\[
\Gamma_{n,\rho} = o\bb{\frac{\mu}{\tau_n(\sqrt{d} + \sqrt{\log(1/\rho)})}}, \qquad \Gamma_{m,\rho/B} = o\bb{\frac{\mu}{\tau_m \kstar \sqrt{\log(1/\rho)}}}.
\]
\bk
Moreover, the confidence radius
\[
\tau_n^{-1}\qnmp = \max\bb{\underbrace{\tau_n^{-1}\qnm}_{\text{non-private bootstrap radius}} \;\; + \;\; \underbrace{\sigma_{\text{priv}} \cdot G}_{\text{private overhead}}, \,\, 0},
\]
where $G \sim \mathcal{N}(0,1)$ and $\sigma_{\text{priv}} \le 2 \tau_m \tau_n^{-1}\mu^{-1}(\Gamma_{m,\rho/B}\kstar e^2 + \Gamma_{n,\rho} e)$ with probability at least $1-O(\rho).$

\end{theorem}

Theorem~\ref{thm:smooth_utility} shows that under mild assumptions on the smooth sensitivity of an estimator of interest, the radius of the confidence region output by Algorithm~\ref{alg:smoothed_quantile_subsample} can be controlled using a high probability upper bound on the smoothed sensitivity of the estimator. 
The proof of Theorem~\ref{thm:smooth_utility} can be found in Appendix~\ref{app:alg4}.




%% file: Algorithms/applications.tex
\section{Applications}
\label{sec:applications}

\subsection{Confidence intervals for population quantiles}
\label{sec:quantile}
Although sample quantiles are consistent, their global sensitivity can be large, making private resampling methods that require small global sensitivity difficult to apply directly. We use $\dpss$ to obtain private confidence intervals for population quantiles, by bounding the smooth sensitivity of the $1-\zeta$ sample quantile in terms of the spacings between appropriate order statistics in the neighborhood of the quantile. The privacy guarantee follows from Theorem~\ref{thm:smooth_utility}, and the utility guarantee reduces to controlling these spacings. The proofs are given in Appendix Section~\ref{app:quantiles}.

\begin{lemma}[Smooth sensitivity of a sample quantile]
\label{lem:ss_quantile}
Let $D=(x_1,\dots,x_n)\in[\ell,u]^n$ and let $x_{(1)}\le \cdots \le x_{(n)}$ be the order statistic of $D$.
Fix a quantile level $1-\zeta\in(0,1)$ and let $h:=\big\lceil (1-\zeta)n\big\rceil$. Then, for any $\beta > 0,$ there exists a $\beta$-smooth upper bound on the local sensitivity of $q_{1-\zeta} = x_{(h)}$. Moreover, for any $C>0$ and $k^\star \ge \max\{\beta^{-1} (\log (C^{-1}(u-\ell)) +3\log n), 0\},$
the smooth sensitivity is bounded as \[\SS_{\beta}^{(1-\zeta)}(\theta, D)
\ \le\
x_{(h+k^\star+1)}-x_{(h-k^\star-1)} + {C}/{n^3}.\]
\end{lemma}



Suppose $\mathcal{D}$ is a distribution supported on $[\ell, u]$ such that its density $f$ satisfies $f(x) \ge \lambda$ for $x$ \textit{around} the $1-\zeta$ quantile (see the statement of Theorem~\ref{thm:ss_density_quantile_bound} for the precise characterization). By using Lemma~\ref{lem:ss_quantile}, we can show that 
for any positive integer $N$, the smooth sensitivity of the sample quantile with $N$ samples is bounded as $\SS_{\beta}^{(1-\zeta)}(\theta, D) \le \frac{6K_N}{N\lambda} + \frac{C}{N^3}$ with probability at least $1-O(\rho)$, where 
\[
K_N = \max\{\beta^{-1} (\log (C^{-1}(u-\ell)) +3\log N), 0\}.
\]
For such distributions $\mathcal{D}$, Theorem~\ref{thm:smooth_utility} applies with $\Gamma_{N,\rho} = \frac{6K_N}{N\lambda} + \frac{C}{N^3}$ and we can show that with high probability, the confidence radius $\tau_n^{-1} \qnmp$ of the algorithm is equal to the non-private confidence radius $\tau_n^{-1}\qnm$ obtained by resampling, plus a private overhead $\sigma_{\text{priv}} \cdot G$, where $G$ is a standard normal and \(\sigma_{\text{priv}} = \tilde{O}(1/\sqrt{mn}).\)
This leads to the bound on the radius in Table~\ref{tab:rw_comparison}.
See Appendix~\ref{app:quantiles} for the exact statement and calculations.

\subsection{Degenerate U-statistics}
\label{sec:degenerate}

Let \(X_1,\ldots,X_N \overset{\mathrm{i.i.d.}}{\sim} \mathcal P\), and let
\(h:\mathcal X^r\to \mathbb R\) be a bounded symmetric kernel of fixed order
\(r\). The corresponding \(U\)-statistic estimates $\theta := \mathbb E[h(X_1,\ldots,X_r)]$ by averaging \(h\) over all \(r\)-tuples of distinct observations:
\[
    U_N
    :=
    \binom{N}{r}^{-1}
    \sum_{1\le i_1<\cdots<i_r\le N}
    h(X_{i_1},\ldots,X_{i_r}).
\]
The variance of a \(U\)-statistic is generally governed by its first Hoeffding projection
\[
    h_1(x)
    :=
    \mathbb E[h(x,X_2,\ldots,X_r)]-\theta .
\]
When \(h_1(X_1)\) is nonzero with positive probability, the U statistic is
called \textit{non-degenerate} and $\operatorname{Var}(U_N) = \Theta(N^{-1}).$ Moreover, $\sqrt{N} \cdot (U_N -\theta)$ is asymptotically normal.

We say that \(h\) is first-order degenerate if
\[
    \mathbb E[h(X_1,\ldots,X_r)\mid X_1] = \theta
    \qquad \text{almost surely},
\]
or equivalently \(h_1(X_1)=0\) almost surely. In this case the first-order
variance contribution in the Hoeffding decomposition vanishes. Hence, for
fixed \(r\), $\operatorname{Var}(U_N)=O(N^{-2})$. Moreover, if the variance of the second-order projection\footnote{The second-order projection is the function $h_2(x,y) = \E[h(x,y,X_3,\dots,X_r)]-\theta-h_1(x)-h_2(y).$} is nonzero, then
\[
    \operatorname{Var}(U_N)=\Theta(N^{-2}).
\]
Thus, in the usual first-order degenerate but second-order non-degenerate
case, the appropriate asymptotic scaling is \(\tau_N=N\). In this case, $N \cdot (U_N-\theta)$ converges to a sum of centered chi-squared random variables.

Degenerate U-statistics arise in hypothesis tests such as Pearson goodness-of-fit tests~\citep{Gregory1977LargeST,Anderson1952AsymptoticTO} and uniformity tests~\citep{diakonikolas2016collision}, model misspecfication~\citep{li2020degenerate, linton2014}, and subgraph counts~\citep{reinert2010random}. In contrast to non-degenerate U-statistics, which converge to a normal distribution, degenerate U-statistics converge weakly to an infinite sum of chi-squared random variables~\citep{lee2019u}. This makes them a perfect candidate for $m$-out-of-$n$ subsampling, since the bootstrap has been shown to be an inconsistent estimator of their limiting distribution~\citep{Bretagnolle1983,arcones_gine1992}.

Our goal in this subsection is to show how $\dpss$ can be used for this class of statistics to obtain a private asymptotically valid confidence interval for \(\theta\). To obtain the required sensitivity control, we use the reweighted U-statistic $\widehat U_N^{\mathrm{rw}}$ of~\cite{chaudhuri2024differentially}. Suppose that $h$ has bounded
additive range, meaning that
\[
\sup_{x,x'} |h(x) - h(x')| \le C
\]
for some $C$.~\cite{chaudhuri2024differentially} show that the estimator $\widehat U_N^{\mathrm{rw}}$ admits a $\beta$-smooth
sensitivity bound satisfying
\begin{equation}
\Pr\left(
SS_\beta(\widehat U_N^{\mathrm{rw}};D_N)
\le \Gamma^U_{N,\eta}
\right)
\ge 1-\eta,
\qquad
\Gamma^U_{N,\eta}
=
O\left(
\frac{r^{3/2}C \log(1/\eta)}{N^{3/2}}
\right).
\label{eq:degen_ss_bound}
\end{equation}

Under Assumption~\ref{ass:limit} and the sensitivity conditions of Theorem~\ref{thm:smooth_utility},
$\dpss$ therefore produces an asymptotically valid confidence interval.
The following lemma records its confidence radius.

\begin{lemma}[Confidence radius for degenerate kernels]
\label{lem:degenerate-ustat-radius}
Let \(X_1,\ldots,X_N\) be IID samples from a distribution $\mathcal{P}$ over $\mathcal{X}$, and let
\(h:\mathcal X^r\to \mathbb R\) be a first-order degenerate, bounded symmetric kernel of order \(r\). Then, the smooth-sensitivity bound of equation~\eqref{eq:degen_ss_bound} holds with
$\Gamma^U_{n,\rho}$ for the original sample and
$\Gamma^U_{m,\rho/B}$ for each resample. Run $\dpss$ with $\tau_N=N$ and $\beta$ as in equation~\ref{eq:beta_choice}. Then, with probability at least $1-O(\rho)$, $\frac{\widetilde Q_{n,m}}{n}
=
\max\left\{
\frac{Q_{n,m}}{n}
+
\sigma_{\mathrm{priv}}G,\,
0
\right\},$
where $G\sim N(0,1)$ and
\ba{
\sigma_{\mathrm{priv}}
\le
\frac{2m}{n\mu}
\left(
\kappa_\star e^2\Gamma^U_{m,\rho/B}
+
e\Gamma^U_{n,\rho}
\right)
=
O\left(
\frac{\kappa_\star r^{3/2}C \log(1/\rho)}
{\mu n\sqrt m}
\right). \label{eq:degenerate_private_radius}
}
\end{lemma}

The proof sketch is given in Appendix~\ref{app:degenerate}. For fixed
$r$, $C$, and privacy parameters, the non-private confidence radius is
of order $n^{-1}$ whereas the privacy overhead is of order
$\widetilde O((n\sqrt m)^{-1})$. Hence, as $m\to\infty$, the privacy
overhead is asymptotically smaller than the nonprivate statistical
error. The square the privacy overhead is
\[
\widetilde O\left(
\frac{\kappa_\star^2 r^3C^2}
{\mu^2mn^2}
\right),
\]
which is the corresponding entry in Table~\ref{tab:rw_comparison}.

%% file: experiments.tex
\section{Experiments}
\label{sec:experiments}

We evaluate private confidence intervals for three estimands: a bounded mean, a first-order degenerate U-statistic, and a population median. We compare BLBQuant~\cite{chadha2024resampling}, the GDP $m$-out-of-$n$ bootstrap~\cite{Dette2025GaussianDP}, and our method using sampling with replacement (WR) and without replacement (WOR).

All experiments use sample sizes
$n\in\{2000,5000,10000,20000,40000,50000\}$, with 400 independent datasets generated for each value of $n$. We use $B=500$ resamples, $N_{\mathrm{mc}}=100$ inner Monte Carlo repetitions for BLBQuant, nominal coverage $0.90$, and total privacy parameter $\mu=8$. We report empirical coverage, mean interval width over finite outputs, and mean runtime.

Privacy budgets are calibrated consistently across methods. The $m$-out-of-$n$ bootstrap assigns a privacy budget of $\mu/\sqrt{2}$ to each of its two blocks. The two pure-DP mechanisms used by BLBQuant are converted to GDP and composed to match the same total GDP parameter.

\begin{figure}[!htbp]
    \centering

    \begin{subfigure}[t]{0.30\linewidth}
        \centering
        \includegraphics[width=\linewidth]{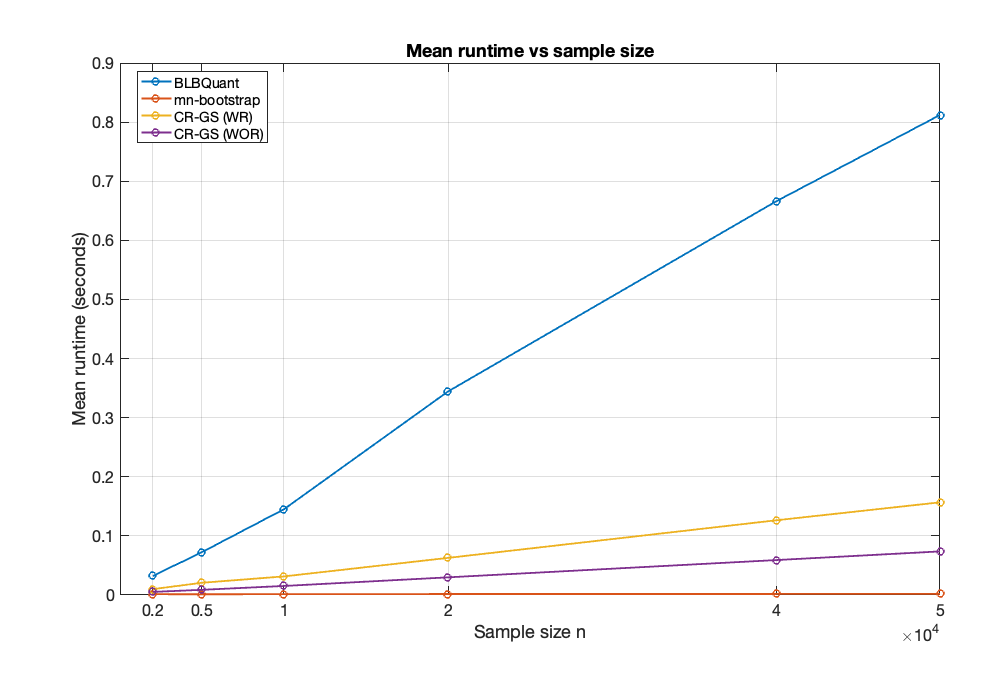}
        \caption{Bernoulli mean.}
        \label{fig:runtime-bernoulli}
    \end{subfigure}
    \hfill
    \begin{subfigure}[t]{0.30\linewidth}
        \centering
        \includegraphics[width=\linewidth]{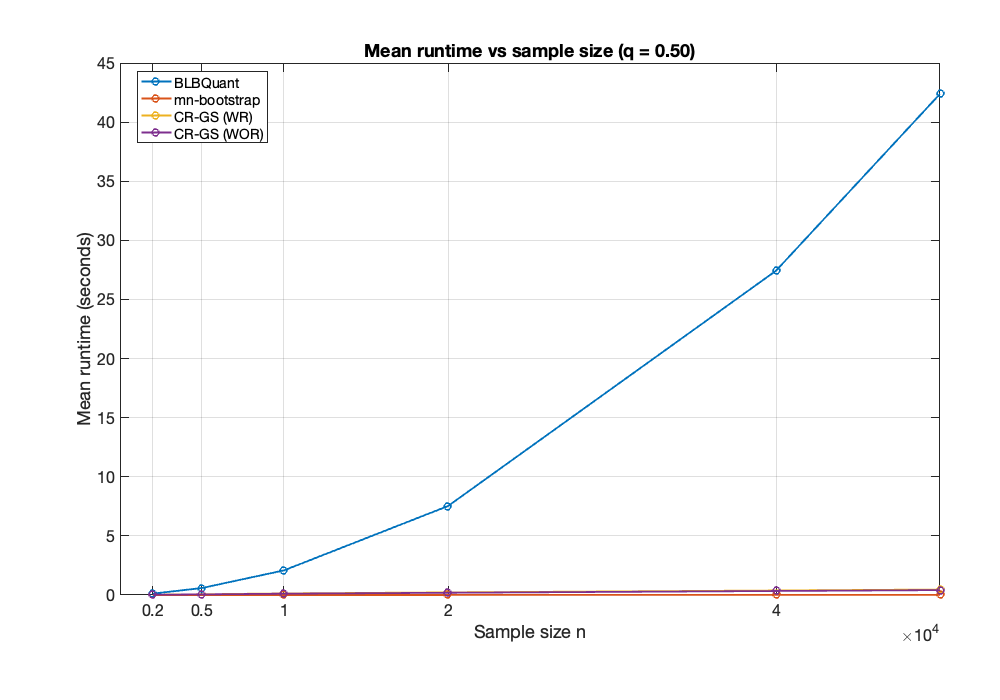}
        \caption{Population median.}
        \label{fig:runtime-median}
    \end{subfigure}
    \hfill
    \begin{subfigure}[t]{0.30\linewidth}
        \centering
        \includegraphics[width=\linewidth]{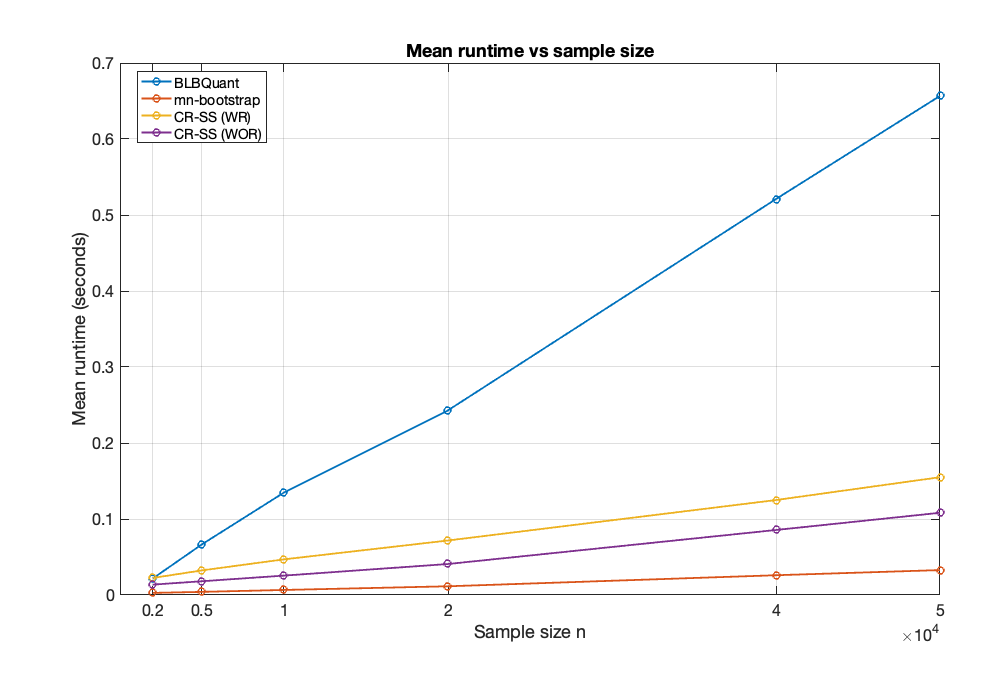}
        \caption{Degenerate U-statistic.}
        \label{fig:runtime-ustat}
    \end{subfigure}

    \caption{\textbf{Mean runtime across three estimation problems.}
    Results are averaged over 400 independent trials for
    \(n\in\{2000,5000,10000,20000,40000,50000\}\).
    We use \(B=500\), \(N_{\mathrm{mc}}=100\), and total privacy
    parameter \(\mu=8\).}
    \label{fig:runtime-main}
\end{figure}

\subsection{Mean estimation}
\label{sec:mean-experiments}

\begin{table}[!htbp]
\centering
\scriptsize
\setlength{\tabcolsep}{2pt}
\renewcommand{\arraystretch}{1.05} 

\begin{subtable}[t]{0.49\linewidth}
\centering
\begin{tabular}{@{}lcccc@{}}
\toprule
Algorithm & cov & \(10^3\!\times\!\mathrm{len}\) & time (ms) & \(n\) \\
\midrule
BLBQuant          & 0.93          & 16.39 & 32.21  & 2000 \\
Dette-Graw & 0.66          & 11.39 & 1.35   & 2000 \\
CR-GS (WR)        & \textbf{0.90} & 15.92 & 9.91   & 2000 \\
CR-GS (WOR)       & \textbf{0.90} & 14.90 & 5.23   & 2000 \\
\midrule
BLBQuant          & \textbf{0.90} & 10.26 & 72.53  & 5000 \\
Dette-Graw & 0.77          & 8.95  & 1.22   & 5000 \\
CR-GS (WR)        & \textbf{0.90} & 10.17 & 20.71  & 5000 \\
CR-GS (WOR)       & 0.86          & 9.46  & 8.78   & 5000 \\
\midrule
BLBQuant          & \textbf{0.89} & 7.18  & 145.11 & 10000 \\
Dette-Graw & 0.82          & 6.63  & 1.40   & 10000 \\
CR-GS (WR)        & \textbf{0.89} & 7.12  & 31.47  & 10000 \\
CR-GS (WOR)       & \textbf{0.88} & 6.76  & 15.31  & 10000 \\
\midrule
BLBQuant          & \textbf{0.92} & 5.08  & 344.29 & 20000 \\
Dette-Graw & \textbf{0.90} & 4.91  & 1.68   & 20000 \\
CR-GS (WR)        & \textbf{0.92} & 5.05  & 62.77  & 20000 \\
CR-GS (WOR)       & \textbf{0.91} & 4.78  & 29.88  & 20000 \\
\midrule
BLBQuant          & \textbf{0.90} & 3.56  & 665.86 & 40000 \\
Dette-Graw & \textbf{0.89} & 3.55  & 1.74   & 40000 \\
CR-GS (WR)        & \textbf{0.90} & 3.57  & 126.34 & 40000 \\
CR-GS (WOR)       & \textbf{0.88} & 3.39  & 59.07  & 40000 \\
\midrule
BLBQuant          & \textbf{0.91} & 3.18  & 812.03 & 50000 \\
Dette-Graw & \textbf{0.89} & 3.12  & 1.81   & 50000 \\
CR-GS (WR)        & \textbf{0.91} & 3.18  & 156.78 & 50000 \\
CR-GS (WOR)       & \textbf{0.89} & 3.04  & 73.74  & 50000 \\
\bottomrule
\end{tabular}
\caption{Bernoulli\((0.05)\) mean.}
\label{tab:bernoulli-results}
\end{subtable}
\hfill
\begin{subtable}[t]{0.49\linewidth}
\centering
\begin{tabular}{@{}lcccc@{}}
\toprule
Algorithm & cov & \(10^3\!\times\!\mathrm{len}\) & time (ms) & \(n\) \\
\midrule
BLBQuant          & \textbf{0.91} & 42.50 & 29.78  & 2000 \\
Dette-Graw & \textbf{0.91} & 42.54 & 1.24   & 2000 \\
CR-GS (WR)        & \textbf{0.89} & 42.38 & 9.01   & 2000 \\
CR-GS (WOR)       & 0.87          & 39.61 & 4.78   & 2000 \\
\midrule
BLBQuant          & \textbf{0.88} & 26.70 & 72.44  & 5000 \\
Dette-Graw & \textbf{0.89} & 26.81 & 1.27   & 5000 \\
CR-GS (WR)        & 0.87          & 26.71 & 20.08  & 5000 \\
CR-GS (WOR)       & 0.86          & 25.22 & 8.76   & 5000 \\
\midrule
BLBQuant          & \textbf{0.88} & 18.88 & 142.71 & 10000 \\
Dette-Graw & \textbf{0.89} & 18.99 & 1.28   & 10000 \\
CR-GS (WR)        & \textbf{0.88} & 18.93 & 31.01  & 10000 \\
CR-GS (WOR)       & 0.87          & 17.85 & 15.13  & 10000 \\
\midrule
BLBQuant          & \textbf{0.92} & 13.29 & 339.56 & 20000 \\
Dette-Graw & \textbf{0.92} & 13.44 & 1.60   & 20000 \\
CR-GS (WR)        & \textbf{0.92} & 13.41 & 62.33  & 20000 \\
CR-GS (WOR)       & \textbf{0.90} & 12.65 & 28.15  & 20000 \\
\midrule
BLBQuant          & \textbf{0.90} & 9.37  & 661.87 & 40000 \\
Dette-Graw & \textbf{0.90} & 9.49  & 1.73   & 40000 \\
CR-GS (WR)        & \textbf{0.90} & 9.47  & 125.23 & 40000 \\
CR-GS (WOR)       & \textbf{0.89} & 9.01  & 58.70  & 40000 \\
\midrule
BLBQuant          & \textbf{0.90} & 8.37  & 817.64 & 50000 \\
Dette-Graw & \textbf{0.90} & 8.50  & 1.85   & 50000 \\
CR-GS (WR)        & \textbf{0.90} & 8.47  & 156.44 & 50000 \\
CR-GS (WOR)       & \textbf{0.89} & 8.07  & 73.65  & 50000 \\
\bottomrule
\end{tabular}
\caption{Uniform\([-1,1]\) mean.}
\label{tab:uniform-results}
\end{subtable}

\caption{\textbf{Results for the bounded-mean experiments.}
Here, cov is empirical coverage, len is the mean interval width over
finite outputs, and time is the mean runtime in milliseconds.
The table reports \(10^3\) times the mean interval width.
Bold coverage values lie in \([0.88,0.92]\) after rounding.}
\label{tab:mean-results}
\end{table}


We consider two bounded distributions representing qualitatively different finite-sample regimes:
\[
    X\sim\operatorname{Bernoulli}(0.05),\qquad
    X\sim\operatorname{Unif}[-1,1].
\]
Their population means are $0.05$ and and $0$, respectively. For observations supported on $[a,b]$, the sample mean has global sensitivity $(b-a)/n$. Moreover, the sample mean is asymptotically normal as long as the variance is finite and non-zero, so the scaling parameter $\tau_n=\sqrt{n}$.

For CR-GS and BLBQuant, we set the miscoverage parameter to $\alpha=0.1$. The $m$-out-of-$n$ bootstrap procedure a $1-2\alpha_{\mathrm{DG}}$ interval, so we use $\alpha_{\mathrm{DG}}=0.05$. For CR-GS, the resample size is
$m_{\mathrm{GS}}=\left\lceil\frac{n}{\log n}\right\rceil.$
The WR variant uses multiplicity-check failure probability $\rho=1/n$. The WOR variant uses the same $m_{\mathrm{GS}}$ and $B$, but has $\kappa^\star=1$ and also does not require a multiplicity check. For $m$-out-of-$n$ bootstrap, we use their recommended resample size $m_{\mathrm{DG}} = \operatorname{round}\!\left\{ \frac{\log(1-1/B)}{\log(1-1/n)}\right\}$~\cite{Dette2025GaussianDP}.

Table~\ref{tab:mean-results} summarizes the results of both mean experiments. All methods approach the nominal coverage of $0.9$ when the sample size $n$ is large. For the Uniform experiment, all methods other than CR-GS without replacement achieve nominal coverage for smaller values of $n$ too, and for the Bernoulli experiment, where the mean is skewed, both BLBQuant and CR-GS achieve near-nominal coverage for smaller values of $n$.

\subsection{Degenerate U-statistic}
\label{sec:ustat-experiments}

\begin{table}[!htbp]
\centering
\scriptsize
\setlength{\tabcolsep}{2pt}
\renewcommand{\arraystretch}{0.92}

\begin{subtable}[t]{0.49\linewidth}
\centering
\begin{tabular}{@{}lrrrr@{}}
\toprule
Algorithm & cov & \(10^3\!\times\!\mathrm{len}\) & time & \(n\) \\
\midrule
BLBQuant          & 0.93          & 101.14 & 117.77   & 2000 \\
Dette-Graw & 0.06          & 45.04  & 1.30     & 2000 \\
CR-SS (WR)        & \textbf{0.88} & 164.35 & 33.38    & 2000 \\
CR-SS (WOR)       & 0.85          & 86.14  & 25.89    & 2000 \\
\midrule
BLBQuant          & 0.94          & 66.94  & 590.56   & 5000 \\
Dette-Graw & 0.12          & 44.15  & 1.45     & 5000 \\
CR-SS (WR)        & 0.81          & 78.84  & 65.41    & 5000 \\
CR-SS (WOR)       & 0.86          & 54.40  & 51.25    & 5000 \\
\midrule
BLBQuant          & 0.93          & 48.34  & 2083.63  & 10000 \\
Dette-Graw & 0.10          & 39.48  & 1.63     & 10000 \\
CR-SS (WR)        & 0.85          & 60.64  & 135.78   & 10000 \\
CR-SS (WOR)       & 0.87          & 39.29  & 113.22   & 10000 \\
\midrule
BLBQuant          & 0.97          & 38.55  & 7523.76  & 20000 \\
Dette-Graw & 0.08          & 28.88  & 2.26     & 20000 \\
CR-SS (WR)        & 0.82          & 35.91  & 221.00   & 20000 \\
CR-SS (WOR)       & \textbf{0.88} & 27.73  & 188.57   & 20000 \\
\midrule
BLBQuant          & 0.98          & 30.18  & 27427.77 & 40000 \\
Dette-Graw & 0.03          & 20.59  & 3.39     & 40000 \\
CR-SS (WR)        & 0.85          & 22.71  & 382.88   & 40000 \\
CR-SS (WOR)       & \textbf{0.88} & 19.54  & 338.46   & 40000 \\
\midrule
BLBQuant          & 0.97          & 26.93  & 42386.22 & 50000 \\
Dette-Graw & 0.04          & 18.43  & 3.87     & 50000 \\
CR-SS (WR)        & 0.85          & 20.52  & 465.20   & 50000 \\
CR-SS (WOR)       & 0.87          & 17.50  & 414.23   & 50000 \\
\bottomrule
\end{tabular}
\caption{Population median.}
\label{tab:median-results}
\end{subtable}
\hfill
\begin{subtable}[t]{0.49\linewidth}
\centering
\begin{tabular}{@{}lrrrr@{}}
\toprule
Algorithm & cov & \(10^4\!\times\!\mathrm{len}\) & time & \(n\) \\
\midrule
BLBQuant          & 1.00          & 11.24 & 21.77  & 2000 \\
Dette-Graw & 1.00          & 7.37  & 2.76   & 2000 \\
CR-SS (WR)        & \textbf{0.90} & 4.08  & 22.43  & 2000 \\
CR-SS (WOR)       & 0.81          & 1.48  & 13.65  & 2000 \\
\midrule
BLBQuant          & 1.00          & 4.90  & 66.71  & 5000 \\
Dette-Graw & 0.99          & 2.65  & 4.12   & 5000 \\
CR-SS (WR)        & 0.87          & 1.21  & 32.39  & 5000 \\
CR-SS (WOR)       & 0.84          & 0.58  & 18.11  & 5000 \\
\midrule
BLBQuant          & 1.00          & 2.63  & 134.81 & 10000 \\
Dette-Graw & 1.00          & 1.23  & 6.63   & 10000 \\
CR-SS (WR)        & 0.86          & 0.56  & 46.92  & 10000 \\
CR-SS (WOR)       & \textbf{0.88} & 0.29  & 25.52  & 10000 \\
\midrule
BLBQuant          & 1.00          & 1.38  & 242.96 & 20000 \\
Dette-Graw & 0.99          & 0.58  & 11.46  & 20000 \\
CR-SS (WR)        & 0.85          & 0.23  & 71.72  & 20000 \\
CR-SS (WOR)       & 0.93          & 0.14  & 40.79  & 20000 \\
\midrule
BLBQuant          & 1.00          & 0.73  & 521.04 & 40000 \\
Dette-Graw & 0.98          & 0.27  & 26.01  & 40000 \\
CR-SS (WR)        & 0.87          & 0.11  & 124.95 & 40000 \\
CR-SS (WOR)       & \textbf{0.89} & 0.07  & 85.54  & 40000 \\
\midrule
BLBQuant          & 1.00          & 0.59  & 656.71 & 50000 \\
Dette-Graw & 1.00          & 0.21  & 32.79  & 50000 \\
CR-SS (WR)        & 0.84          & 0.07  & 155.00 & 50000 \\
CR-SS (WOR)       & \textbf{0.92} & 0.06  & 108.17 & 50000 \\
\bottomrule
\end{tabular}
\caption{Degenerate U-statistic.}
\label{tab:ustat-results}
\end{subtable}

\caption{\textbf{Results for the median and degenerate U-statistic experiments.}
Here, cov is empirical coverage, len is mean interval width over finite
outputs, and time is mean runtime in milliseconds. The median panel
reports \(10^3\) times the interval width, while the U-statistic panel
reports \(10^4\) times the interval width. For CR-SS (WR) in the
U-statistic experiment, coverage is conditional on return. Bold
coverage values lie in \([0.88,0.92]\) after rounding.}
\label{tab:median-ustat-results}
\end{table}

We generate
$X_1,\ldots,X_n\stackrel{\mathrm{i.i.d.}}{\sim}\operatorname{Unif}[-1/2,1/2]$
and use the degree-two kernel $h(x,y)=xy$. The resulting U-statistic is
\[
    U_n
    =
    \binom{n}{2}^{-1}\sum_{i<j}X_iX_j
    =
    \frac{\left(\sum_iX_i\right)^2-\sum_iX_i^2}{n(n-1)}.
\]
Because $\mathbb{E}[X_i]=0$, the target parameter is
$\theta=(\mathbb{E}[X_i])^2=0$. The statistic is first-order degenerate, so the appropriate scaling is $\tau_n=n$.

For CR-SS, we use $m=\lceil {n}/{\log n} \rceil, \delta=10^{-6},$ and $\rho=1/n,$ together with the product-statistic-specific smooth-sensitivity bound. Because the statistic is degenerate, the resampling pivots are constructed on the $m$-scale, and the resulting interval radii are converted back to the $n$-scale.

Table~\ref{tab:median-ustat-results} summarizes the results of the experiment. We observe that $m$-out-of-$n$ bootstrap undercovers significantly, because the global sensitivity of the median is a constant, which is not consistent with what they expect in their experiments. BLBQuant overcovers, whereas our methods slightly undercover.

\subsection{Median}
\label{sec:median-experiments}

We generate
$X_1,\ldots,X_n\stackrel{\mathrm{i.i.d.}}{\sim}\operatorname{Unif}[-1,1]$.
The target parameter is the population median, which equals $0$. Since the empirical median is asymptotically normal, we use the scaling $\tau_n=\sqrt{n}$.

For CR-SS, we use $m=\lceil {n}/{\log n} \rceil, \delta=10^{-6},$ and $\rho=1/n,$ together with the quantile smooth-sensitivity bound. For WR, the miscoverage parameter is set to $0.09$, reserving $\rho=0.01$ for the multiplicity-check failure probability. For WOR, the miscoverage parameter is $0.10$. The global sensitivity of the empirical median on $[-1,1]$ is $2$, which is the value used by $m$-out-of-$n$ bootstrap.

Table~\ref{tab:median-ustat-results} summarizes the results of the experiment. We observe that $m$-out-of-$n$ bootstrap and BLBQuant both overcover, with a coverage close to 1. Since our methods adapt to the natural $\tau_n = n$ scaling of a degenerate U-statistic, we obtain coverages closer to the nominal 0.9 coverage. 

%% file: conclusion.tex
\section{Conclusion}
\label{sec:conclusion}
We provide GDP guarantees for private confidence regions for a large family of estimators arising from widely used resampling procedures such as subsampling and $m$-out-of-$n$ sampling. We provide a simple algorithm that privatizes the quantile instead of privatizing estimators computed on each resample. This leads to a logarithmic dependence on $B$ instead of the $\sqrt{B}$ dependence incurred in previous work. Our estimators \textit{do not need to be asymptotically normal} or have small global sensitivity, making it widely applicable. We provide both utility and \textit{nonasymptotic} privacy guarantees. Overall, we develop private uncertainty quantification methods for common resampling schemes that sidestep the computational and privacy costs of privatizing many intermediate resample statistics.

\textbf{AI use statement:} The authors used AI as a support tool while preparing this paper. AI assisted with some of the proofs, helped with the implementation, and polishing the writing. The authors reviewed and edited all AI-assisted content and take responsibility for the paper's accuracy, originality, and final content.


%% file: appendix.tex






\appendix
\addtocontents{toc}{\protect\setcounter{tocdepth}{1}}

\input{map}

\section{Auxiliary Results}
\label{app:auxiliary}

\begin{proposition}[$\ell_2$-norm concentration of standard Gaussian~\citep{wainwright2019high}]\label{prop:l2_concentration_gaussian} If $Z \sim \mathcal{N}(0, \id_d)$, then for any $\beta \in (0,1)$,
\[
\Pr(\norm{Z}_2 \le \sqrt{d} + \sqrt{2 \log (1/\beta)}) \ge 1-\beta.
\]
\end{proposition}

\begin{proposition}[Local Sensitivity of Order Statistics]\label{lem:local_order_statistic}
Let $1 \le k \le B$. Let $\mathcal{A}_1, \mathcal{A}_2, \dots, \mathcal{A}_B : \mathcal{X}^m \to \mathbb{R}$ be arbitrary mechanisms and define $g_k(D)$ to be the $k$th largest element in $\{\mathcal{A}_1(D), \dots, \mathcal{A}_B(D)\}$. Then,
\[
\LS(g_k,D) \le \max_{i \in [B]} \LS(\mathcal{A}_i,D).
\]
\end{proposition}

\begin{proof}
Fix a neighboring dataset $D' \sim D$, and let $\delta := \max_{i \in [B]}
\left|\mathcal A_i(D)-\mathcal A_i(D')\right|.$
At least $k$ of the values $\mathcal A_i(D)$ are at least $g_k(D)$.
Their corresponding values at $D'$ are therefore at least
$g_k(D)-\delta$, so $g_k(D') \ge g_k(D)-\delta.$

Interchanging $D$ and $D'$ gives $g_k(D) \ge g_k(D')-\delta$. Therefore, $|g_k(D)-g_k(D')|\le \delta \le \max_{i \in [B]} LS(\mathcal A_i,D).$
Taking the supremum over all $D' \sim D$ yields
\[
LS(g_k,D) \le \max_{i \in [B]} LS(\mathcal A_i,D).
\]
\end{proof}

\section{More details on subsampling validity results for m-out-of-n and subsampling}
\label{sec:resampDetails}
Let $F_n(t):=\Pr(\tau_n d(\hatn,\theta)\leq t)$ and let $F(t) := \lim_{n \to \infty} F_n(t)$ denote the CDF of the limiting distribution.
\begin{assumption}\label{ass:limit:appendix}
    We assume that 
    \begin{enumerate}
        \item $F$ is non-degenerate and continuous.
        \item $\sup_t |F_n(t)-F(t)|\rightarrow 0.$ 
        \item $q_n=c_{1-\alpha}+o(1)$ where $q_n$ is the true $1-\alpha$ quantile of the distribution of $\tau_n d(\hatn,\theta)$ and $c_{1-\alpha}$ is the $1-\alpha$ quantile of the limiting distribution.
        \item $\tau_m/\tau_n\rightarrow 0$, $m/n\rightarrow 0$, $B\rightarrow\infty$.
        
        \item (Resampling validity) $\Pr(\tau_n d(\widehat{\theta}_n,\theta)\leq \qnm)=1-\alpha+o_n(1)$, where $\qnm$ denotes the scaled quantile defined in Algorithm~\ref{alg:gdp_quantile_subsample}

    \end{enumerate}
\end{assumption}
\begin{remark}
     Let $J_n(P)$ denote the limit law of $\tau_n(\hatn-\theta)$.  The standard assumption for subsampling validity~\cite{PolitisRomanoWolf1999WeakConvergenceEmpiricalMeasures,PolitisRomanoWolf2001SubsamplingTheory} is that there exists a limit law $J(P)$ so that $J_n(P)$ converges to $J(P)$ weakly as $n\rightarrow \infty$. Note that the continuous mapping theorem shows that this immediately implies that there exists a limit law of the sampling distribution of $\tau_n\|\hatn-\theta\|$ (Assumption~\ref{ass:limit:appendix} (1) and (2). All that remains to be checked is Assumption~\ref{ass:limit:appendix} (3) and (5). Lemma 21.2 in~\cite{vanderVaart1998Asymptotic} establishes that if $F^{-1}(.)$ is continuous at $1-\alpha$, the $1-\alpha$ quantile $q_n$ converges to the true quantile $c_{1-\alpha}$. As for Assumption~\ref{ass:limit:appendix} (5), note that Theorem 2.1 (iv) in~\cite{PolitisRomanoWolf2001SubsamplingTheory} directly yields (5) for infinite $B$. When $B\rightarrow \infty$, one can show condition (5) via the DKW inequality and quantile convergence (see Lemma~\ref{lem:coverage_estimated_quantile} in the Appendix). A similar argument also applies to $m$-out-of-$n$ resampling. For simplicity, we move it to the appendix.
\end{remark}

\begin{remark}[With-replacement $m$-out-of-$n$ validity]
Let $J_n(P)$ denote the law of $\tau_n(\hat\theta_n-\theta(P))$, and suppose $J_n(P)\Rightarrow J(P)$ for a
nondegenerate limit law $J(P)$ with continuous CDF $F$.
For with-replacement $m$-out-of-$n$ resampling, let $\hat\theta_{m}^{\ast}$ be the estimator computed on a size-$m$
bootstrap resample drawn i.i.d.\ from $\hat P_n$, and define the \emph{ideal} (infinite-resample) conditional CDF
\[
B_{m,n}(t)\;:=\;\Pr^\ast\!\Big(\tau_m\, d(\hat\theta_{m}^{\ast},\theta(P)) \le t \,\Big|\, D_n\Big),
\]
where $\Pr^\ast(\cdot\,|\,D_n)$ is probability over the resampling randomness conditional on the data.
Under the standard $m$-out-of-$n$ conditions $m\to\infty$, $m/n\to 0$, and $\tau_m/\tau_n\to 0$, together with the
regularity assumptions in Theorem~2 of \cite{RESAMPLING} (see in particular (3.10)--(3.14)),
the $m$-out-of-$n$ bootstrap is consistent in the sense that
\[
\sup_{t\in\mathbb{R}} \big| B_{m,n}(t) - F(t) \big| \;\xrightarrow{p}\; 0.
\]
Consequently, if $q_{n,m}(1-\alpha):=\inf\{t: B_{m,n}(t)\ge 1-\alpha\}$ and $F$ is continuous at $c_{1-\alpha}$,
then $q_{n,m}(1-\alpha)\to c_{1-\alpha}$, and
\[
\Pr\!\Big(\tau_n\, d(\hat\theta_n,\theta(P)) \le q_{n,m}(1-\alpha)\Big)\;=\;1-\alpha+o_n(1),
\]
by the standard quantile-continuity argument (e.g.\ Lemma~21.2 in \cite{Vaart_1998} or Theorem~1 in
\cite{bm}).
Finally, Assumption~\ref{ass:limit:appendix} uses the empirical quantile $Q_{n,m}$ from $B$ i.i.d.\ resamples.
Let $\widehat B_{m,n,B}$ denote the empirical CDF of $\{\tau_m d(\hat\theta_{m,i}^\ast,\theta(P))\}_{i=1}^B$.
Conditional on $D_n$, by the DKW inequality,
\[
\sup_t\big|\widehat B_{m,n,B}(t)-B_{m,n}(t)\big|=O_p(B^{-1/2}),
\]
so $Q_{n,m}-q_{n,m}(1-\alpha)=o_p(1)$ provided $B\to\infty$ (and the target quantile is a continuity point).

\end{remark}

\begin{lemma}[Coverage with an estimated quantile cutoff]
\label{lem:coverage_estimated_quantile}
Let $T_n$ be a real-valued statistic with $T_n \stackrel{d}{\rightarrow} T$, where $T$ has CDF $F$.
Fix $\alpha\in(0,1)$ and let $q := F^{-1}(1-\alpha)=\inf\{t:F(t)\ge 1-\alpha\}$.
Assume that $F$ is continuous at $q$ (equivalently, $\Pr(T=q)=0$).
Let $Q_n$ be any (possibly random) cutoff such that $Q_n \xrightarrow{p} q$.
Then
\[
\Pr\!\big(T_n \le Q_n\big)\ \to\ 1-\alpha,
\qquad\text{and}\qquad
\Pr\!\big(T_n \le Q_n\big)-\Pr\!\big(T_n \le q\big)\ \to\ 0.
\]
\end{lemma}

\begin{proof}
Fix $\eta>0$. Using $\{T_n\le Q_n\}\subseteq \{T_n\le q+\eta\}\cup\{Q_n>q+\eta\}$ and
$\{T_n\le q-\eta\}\subseteq \{T_n\le Q_n\}\cup\{Q_n<q-\eta\}$ gives
\[
\Pr(T_n\le q-\eta)-\Pr(Q_n<q-\eta)
\ \le\ 
\Pr(T_n\le Q_n)
\ \le\
\Pr(T_n\le q+\eta)+\Pr(Q_n>q+\eta).
\]
Since $Q_n\xrightarrow{p}q$, the tail probabilities on the right vanish as $n\to\infty$.
By definition of convergence in distribution, $\Pr(T_n\le q\pm\eta)\to F(q\pm\eta)$.
Taking $\liminf$ and $\limsup$ yields
$F(q-\eta)\le \liminf_n \Pr(T_n\le Q_n)\le \limsup_n \Pr(T_n\le Q_n)\le F(q+\eta)$.
Letting $\eta\downarrow 0$ and using continuity of $F$ at $q$ gives $\Pr(T_n\le Q_n)\to F(q)=1-\alpha$.
Finally, $\Pr(T_n\le q)\to F(q)$ by definition, so the difference also converges to $0$.
\end{proof}


\subsection{Proof of Proposition~\ref{prop:kstar_choice}}

Fix $i\in[B]$ and $j\in[n]$. Since $S_i$ consists of $m$ i.i.d.\ draws from $[n]$ with replacement,
\[
\kappa_{i,j}\ \sim\ \mathrm{Bin}\!\left(m,\frac1n\right),
\qquad \mu := \E[\kappa_{i,j}] = \frac{m}{n}.
\]
By a union bound over all $(i,j)\in[B]\times[n]$,
\begin{equation}
\label{eq:union}
\Pr\!\big[\mathcal E^c_{\kstar} \big]
\ \le\ 
\sum_{i=1}^B\sum_{j=1}^n \Pr[\kappa_{i,j}\ge k]
\ =\ 
Bn \cdot \Pr\!\left[\mathrm{Bin}\!\left(m,\frac1n\right)\ge k\right].
\end{equation}

Let $X\sim \mathrm{Bin}(m,1/n)$ with mean $\mu=m/n$.
For any integer $k\ge 1$, a standard Chernoff bound gives
\begin{equation}
\label{eq:bin-tail}
\Pr[X\ge k]
\ \le\
\Big(\frac{e\mu}{k}\Big)^k
\ =\ 
\exp\!\Big(-k\log\frac{k}{e\mu}\Big).
\end{equation}

Recall
\[
A := \log\!\Big(\frac{Bn}{\rho}\Big),\qquad
A_0 := \max\{A,1\},\qquad
M := \log\!\Big(\frac{n}{m}\Big)+\log A_0,
\qquad
\kstar := \left\lceil \frac{4A_0}{M}\right\rceil.
\]
Let $k=\kstar$. Since $k\ge 4A_0/M$ and $\mu=m/n$, we have
\begin{align}
\log\frac{k}{e\mu}
&=
\log k + \log\frac{n}{m} - 1 \notag\\
&\ge
\log\Big(\frac{4A_0}{M}\Big)+\log\frac{n}{m}-1 \notag\\
&=
\big(\log A_0 + \log\frac{n}{m}\big) + \log 4 - \log M - 1 \notag\\
&=
M + \log 4 - \log M - 1.
\label{eq:lowerlog}
\end{align}
Define
\[
g(M) := M + \log 4 - \log M - 1 = (M-\log M) + (\log 4 - 1).
\]
We lower bound $g(M)$ in two regimes.

\smallskip
\noindent\emph{Case 1: $M \le 4\log 4$.}
Since $g'(M)=1-\frac{1}{M}$ and $g''(M)=\frac{1}{M^2}>0$, the function $g$ is convex
and attains its minimum at $M=1$, where $g(1)=\log 4$.
Hence $g(M)\ge \log 4$ for all $M>0$, and in particular for $M\le 4\log 4$,
\[
k\log\frac{k}{e\mu}
\ \ge\
\frac{4A_0}{M}\cdot g(M)
\ \ge\
\frac{4A_0}{M}\cdot \log 4
\ \ge\
A_0,
\]
where the last inequality uses $M\le 4\log 4$.

\smallskip
\noindent\emph{Case 2: $M > 4\log 4$.}
In this case $M\ge 2$, so $\log M \le M/2$, and thus
\[
g(M)= M + \log 4 - \log M - 1 \ \ge\ M - \log M - 1 \ \ge\ M/2 - 1.
\]
Since $M/2-1>M/4$,
\[
k\log\frac{k}{e\mu}
\ \ge\
\frac{4A_0}{M}\cdot (M/2-1)
\ \ge\
\frac{4A_0}{M}\cdot (M/4)
\ =\
A_0.
\]

\smallskip
Combining the two cases, we have shown
\begin{equation}
\label{eq:exp}
k\log\frac{k}{e\mu}\ \ge\ A_0\ \ge\ A.
\end{equation}
Plugging~\eqref{eq:exp} into~\eqref{eq:bin-tail} yields
\[
\Pr\!\left[\mathrm{Bin}\!\left(m,\frac1n\right)\ge k\right]
\le e^{-A}
= \frac{\rho}{Bn}.
\]
Finally, substitute this into~\eqref{eq:union} to obtain
\[
\Pr[\mathcal E^c_{\kstar}] \ \le\ Bn\cdot \frac{\rho}{Bn} \ =\ \rho,
\]
as claimed.

\section{Comparison with related works}\label{sec:appendix:compare}
\input{comparison}

\section{Privacy and Utility for \dpgs}
\label{app:global}

\begin{lemma}[Privacy Guarantee] \label{lem:global_privacy_guarantee}
Algorithm~\ref{alg:gdp_quantile_subsample} is $\mu$-GDP.
\end{lemma}

\begin{proof}
First, the probability that algorithm $\texttt{MCS}$ outputs $\bot$ is independent of the data. It suffices to show that the algorithm is $\mu$-GDP conditioned on the event that $(S_1, \dots, S_B) \neq \bot$.

Since $\hatn = \theta(D) \in \R^d$ has global sensitivity $\Delta_n,$ Proposition~\ref{prop:gaussian_mechanism} implies that the statistic $\tiln$ defined as \[\tiln(D) = \hatn(D) + \frac{\Delta_n}{\mu/\sqrt{2}} Z_2, \;\;\; Z_2 \sim \mathcal{N}(0, \id_d)\]
is $\mu/\sqrt{2}$-GDP. 

Moreover, each $d(\theta_i^{*}, \hatn)$ has global sensitivity $\kstar \Delta_m + \Delta_n$; indeed, each data point $X_j \in D$ appears at most $\kstar$ times in $D[S_i]$, so the global sensitivity of $\theta_i^{*}$ is bounded by $\kstar \Delta_m$.

By Proposition~\ref{lem:local_order_statistic}, the global sensitivity of $\texttt{Quantile}_{1-\alpha}(d(\theta_1^*,\hatn),\dots,d(\theta_B^*,\hatn))$ is also bounded above by $\kstar \Delta_m + \Delta_n.$ Therefore,
\[
\qnmp = \max\bb{\qnm + \frac{\sqrt{2} \tau_m (\kstar \Delta_m + \Delta_n)}{\mu} \cdot Z_1, 0},
\]
where $Z_1 \sim \mathcal{N}(0, 1),$ is $\mu/\sqrt{2}$-GDP. By basic composition (equation~\eqref{eq:pure_gdp_composition}), Algorithm~\ref{alg:gdp_quantile_subsample} is $\mu$-GDP.
\end{proof}
\begin{lemma}\label{lem:global_utility}
 Let \[\sigma_m=\frac{\sqrt{2}\tau_m (\kstar\Delta_m+\Delta_n)}{\mu} \qquad \text{and} \qquad \nu_n = \frac{\sqrt{2}\Delta_n}{\mu}\] 
 as in Algorithm~\ref{alg:gdp_quantile_subsample}. Suppose
    \begin{itemize}
    \item [(i)] assumption~\ref{ass:limit:appendix} holds,
    \item [(ii)] the failure probability $\rho = o_n(1),$ and 
    \item [(iii)] $\gamma = o_n(1)$, where
    \[
    \gamma := \sigma_{m}\sqrt{2\log (1/\rho)} +  \tau_n\nu_n(\sqrt{d}+\sqrt{2\log (1/\rho)}).
    \]
\end{itemize}
 Then, output $(\tiln, \qnmp)$ of Algorithm~\ref{alg:gdp_quantile_subsample} satisfies
\[
\Pr(\tau_n d(\theta, \tiln) \ge \qnmp) = \alpha + o_n(1).
\]
\end{lemma}

\begin{proof}
Let $\mathcal{E}_n$ denote the event that Algorithm~2 does not abort. By
Proposition~4,
\[
    \mathbb{P}(\mathcal{E}_n^c) \leq \rho = o_n(1).
\]
For the purposes of the proof, regard all $B$ resamples as being drawn before
the multiplicity check, and define $Q_{n,m}$ from these resamples regardless
of whether $\mathcal{E}_n$ occurs. Sample $Z_1$ and $Z_2$ independently and
define the auxiliary outputs
\[
    \bar{\theta}_n := \widehat{\theta}_n+\nu_n Z_2,
    \qquad
    \bar{Q}_{n,m}:=\max\{Q_{n,m}+\sigma_m Z_1,0\}.
\]
On $\mathcal{E}_n$, these auxiliary outputs agree with the outputs
$(\widetilde{\theta}_n,\widetilde{Q}_{n,m})$ of Algorithm~2. On
$\mathcal{E}_n^c$, extend the output of Algorithm~2 by any fixed value; the
choice of extension affects the probability in the statement by at most
$\rho=o_n(1)$.

Define
\[
    T_n:=\tau_n d(\theta,\widehat{\theta}_n),
    \qquad
    \bar{T}_n:=\tau_n d(\theta,\bar{\theta}_n),
\]
and let
\[
    q_{1-\alpha}:=F^{-1}(1-\alpha).
\]
By triangle inequality,
\[
    |\bar{T}_n-T_n|
    \leq \tau_n d(\bar{\theta}_n,\widehat{\theta}_n)
    =\tau_n\nu_n\|Z_2\|_2.
\]
Moreover, since $Q_{n,m}\geq 0$ and projection onto $[0,\infty)$ is
$1$-Lipschitz,
\[
    |\bar{Q}_{n,m}-Q_{n,m}|
    \leq \sigma_m|Z_1|.
\]

By Gaussian concentration,
\[
    \mathbb{P}\!\left(\|Z_2\|_2>\sqrt{d}+\sqrt{2\log(1/\rho)} \right)\leq \rho,
    \qquad
    \mathbb{P}(|Z_1|> \sqrt{2\log(1/\rho)})\leq 2\rho.
\]
Consequently,
\[
\begin{aligned}
&\mathbb{P}\left(
    \tau_n\nu_n\|Z_2\|_2+\sigma_m|Z_1|>\gamma
\right)                                                     \\
&\qquad\leq
\mathbb{P}\!\left(\|Z_2\|_2>\sqrt{d}+\sqrt{2\log(1/\rho)}\right)
+\mathbb{P}(|Z_1|>\sqrt{2\log(1/\rho)})
\leq 3\rho=o_n(1).
\end{aligned}
\]
Since $\gamma=o_n(1)$, it follows that
\[
    \bar{T}_n-T_n=o_p(1),
    \qquad
    \bar{Q}_{n,m}-Q_{n,m}=o_p(1).
\]

By Assumption~2,
\[
    T_n \xrightarrow{d} T,
    \qquad
    Q_{n,m}\xrightarrow{p}q_{1-\alpha},
\]
where $T$ has CDF $F$. By Slutsky's theorem,
\[
    \bar{T}_n\xrightarrow{d}T,
    \qquad
    \bar{Q}_{n,m}\xrightarrow{p}q_{1-\alpha}.
\]

To handle the possibly dependent random cutoff, fix $\eps>0$. We have
\[
\mathbb{P}(\bar{T}_n<q_{1-\alpha}-\eps)
-\mathbb{P}(\bar{Q}_{n,m}<q_{1-\alpha}-\eps)  
\leq \mathbb{P}(\bar{T}_n<q_{1-\alpha}-\eps, \bar{Q}_{n,m} \ge q_{1-\alpha}-\eps) \leq
\mathbb{P}(\bar{T}_n<\bar{Q}_{n,m}).
\]

Now, take $\liminf$ on both sides. Since $\bar{Q}_{n,m} \xrightarrow{p} q_{1-\alpha},$
\[
\mathbb{P}(\bar{Q}_{n,m}<q_{1-\alpha}-\eps) \rightarrow 0.
\]
Also, since $F$ is continuous,
\[
\mathbb{P}(\bar{T}_n<q_{1-\alpha}-\eps) \rightarrow \mathbb{P}(T<q_{1-\alpha}-\eps) = F(q_{1-\alpha}-\eps).
\]
It follows that
\[
F(q_{1-\alpha}-\eps) \le \liminf_{n\to\infty}
\mathbb{P}(\bar{T}_n<\bar{Q}_{n,m}).
\]
Similarly, it can be shown that
\[
\limsup_{n\to\infty}
\mathbb{P}(\bar{T}_n<\bar{Q}_{n,m}) \le F(q_{1-\alpha}+\eps).
\]
Letting $\eps\downarrow0$ yields
\[
    \mathbb{P}(\bar{T}_n<\bar{Q}_{n,m})
    =F(q_{1-\alpha})+o_n(1)
    =1-\alpha+o_n(1).
\]
Hence,
\[
    \mathbb{P}(\bar{T}_n\geq\bar{Q}_{n,m})
    =\alpha+o_n(1).
\]

Finally, the actual and auxiliary outputs agree on the event $\mathcal{E}_n$, and
\[
\begin{aligned}
\bigg|
&\mathbb{P}\!\left(
    \tau_n d(\theta,\widetilde{\theta}_n)
    \geq \widetilde{Q}_{n,m}
\right)
-
\mathbb{P}\!\left(
    \bar{T}_n\geq\bar{Q}_{n,m}
\right)
\bigg|
\leq \mathbb{P}(\mathcal{E}_n^c)
\leq \rho.
\end{aligned}
\]
Since $\rho=o_n(1)$, we conclude that
\[
    \mathbb{P}\!\left(
        \tau_n d(\theta,\widetilde{\theta}_n)
        \geq \widetilde{Q}_{n,m}
    \right)
    =\alpha+o_n(1).
\]
\end{proof}

\section{Proofs of Proposition,~\ref{prop:approx_gdp_equiv},~\ref{prop:approx_gdp_composition} and Lemma~\ref{lem:ptr_gdp_privacy}}
\label{app:ptr_gdp_proof}

\subsection{Proof of Proposition~\ref{prop:approx_gdp_equiv}}

For any distribution $\nu$ and set $A$, we let $\nu(A)$ denote the probability $\Pr_{x \sim \nu}\bb{x \in A}$.

Let $P_0:=\mathcal N(0,1)$ and $R_0:=\mathcal N(\mu,1),$ viewed as distributions on the enlarged space
$\overline{\R} := \R \cup\{\bot\}$, with zero mass assigned to
the atom $\bot$. Clearly, $T(P_0,R_0)=G_\mu.$

We first construct a distribution $R_\delta$ such that
\ba{
    T(P_0,R_\delta)= G_{\mu,\delta}
    \qquad\text{and}\qquad
    \TV(R_\delta,R_0)=\delta. \label{eq:R_delta}
}
If $\delta=0$, simply take $R_\delta=R_0$. Suppose therefore that
$\delta>0$. Define $c_\delta:=\mu+\Phi^{-1}(\delta),$ so that
\[
    R_0((-\infty,c_\delta])=\delta.
\]
Define $R_\delta$ by moving this lower-tail mass to the additional point
$\bot$:
\[
    R_\delta(A)
    :=
    R_0\bigl(A\cap(c_\delta,\infty)\bigr)
    +
    \delta\,\mathbf 1\{\bot\in A\}.
\]
It follows immediately that $\TV(R_\delta,R_0)=\delta.$

Next, we  calculate $T(P_0,R_\delta)$. The likelihood ratio of $R_\delta$
with respect to $P_0$ is infinite at $\bot$, zero on
$(-\infty,c_\delta]$, and equal to
\[
    \frac{\frac{1}{\sqrt{2\pi}} \exp\bb{-\frac{(x-\mu)^2}{2}}}{\frac{1}{\sqrt{2\pi}} \exp\bb{-\frac{x^2}{2}}} = \exp\left(\mu x-\frac{\mu^2}{2}\right)
\]
on $(c_\delta,\infty)$, which is non-decreasing in $x$. By the Neyman--Pearson lemma, an optimal test first rejects at $\bot$ and then rejects for sufficiently large values of $x$. 

Fix $\alpha\in[0,1]$ and let $z_\alpha:=\Phi^{-1}(1-\alpha).$ 

\begin{itemize}
    \item [(i)] If $z_\alpha\geq c_\delta$, the rejection region
\[
    \{\bot\}\cup[z_\alpha,\infty)
\]
has Type I error $\alpha$ and Type II error
\bas{
R_\delta((-\infty,z_\alpha)) = R_0((c_\delta,z_\alpha)) = \Phi(z_\alpha-\mu)-\delta = G_\mu(\alpha)-\delta.
}
The condition $z_\alpha\geq c_\delta$ is equivalent to
$G_\mu(\alpha)\geq\delta$.

    \item [(ii)] If $z_\alpha<c_\delta$, rejecting on
\[
    \{\bot\}\cup(c_\delta,\infty)
\]
has Type I error at most $\alpha$ and Type II error zero. Consequently,
\[
    T(P_0,R_\delta)(\alpha)
    =
    \max\{G_\mu(\alpha)-\delta,0\}
    =
    G_{\mu,\delta}(\alpha).
\]
\end{itemize}
This proves equation~\eqref{eq:R_delta}. Next, we prove the two implications of the proposition.

$(\Rightarrow)$
Suppose $T(P,Q)\geq G_{\mu,\delta}.$ By equation~\eqref{eq:R_delta}, 
\[T(P,Q)\geq T(P_0,R_\delta).\]
By Blackwell-Sherman-Stein informativeness theorem~\cite{blackwell1951comparison, sherman1951theorem, stein1951notes, blackwell1953equivalent}, there exists a randomized algorithm $\mathcal{K}$ such that $\mathcal{K}(P_0) = P$ and $\mathcal{K}(R_\delta) = Q$. 

Let $R:=\mathcal{K}(R_0).$ By the data-processing inequality for trade-off functions,
\bas{
    T(P,R) = T(\mathcal{K}(P_0),\mathcal{K}(R_0)) \geq T(P_0,R_0) = G_\mu.
}
Moreover, by the data-processing inequality,
\bas{
    \TV(Q,R) = \TV(\mathcal{K}(R_\delta),\mathcal{K}(R_0)) \leq \TV(R_\delta,R_0) = \delta.
}

$(\Leftarrow)$
Suppose there exists a distribution $R$ such that
\[
    T(P,R)\geq G_\mu
    \qquad\text{and}\qquad
    \TV(Q,R)\leq\delta.
\]
Fix $\alpha\in[0,1]$, and let $\phi$ be any possibly randomized test
satisfying $\mathbb E_P[\phi]\leq\alpha.$ The variational characterization of total variation gives $\left| \mathbb E_Q[\phi]-\mathbb E_R[\phi] \right| \leq\delta.$ Therefore,
\bas{
    1-\mathbb E_Q[\phi] \geq
    1-\mathbb E_R[\phi]-\delta \geq
    T(P,R)(\alpha)-\delta \geq
    G_\mu(\alpha)-\delta.
}
Taking the infimum over all tests with Type I error at most $\alpha$
gives
\[
    T(P,Q)(\alpha)\geq G_\mu(\alpha)-\delta.
\]
Since a trade-off function is nonnegative, $T(P,Q)(\alpha) \geq \max\{G_\mu(\alpha)-\delta,0\} = G_{\mu,\delta}(\alpha).$

\subsection{Proof of Proposition~\ref{prop:approx_gdp_composition}}

We restate the proposition for $k=2$ here.

\begin{proposition}\label{prop:approx_gdp_composition_two}
Suppose $\mathcal M_1: {X} \to {Y}_1$ is $(\mu_1,\delta_1)$-GDP and $\mathcal M_2 : {X} \times {Y}_1 \to {Y}_2$ is such that $M_2(\cdot, y_1)$ is $(\mu_2,\delta_2)$-GDP for any $y_1 \in {Y}_1$. Then the mechanism $\mathcal M(D)=(Y_1 ,Y_2)$, where $Y_1 = \mathcal{M}_1(D)$ and $Y_2 = \mathcal{M}_2(D, Y_1),$
is $\left(
\sqrt{\mu_1^2+\mu_2^2},
\;
\delta_1+\delta_2-\delta_1\delta_2
\right)$-GDP.
\end{proposition}

Note that Proposition~\ref{prop:approx_gdp_composition} follows by a repeated application of Proposition~\ref{prop:approx_gdp_composition_two}.

\begin{proof}
Fix neighboring datasets \(D\) and \(D'\). Let \(P_1\) and \(Q_1\) be
the distributions of \(\mathcal M_1(D)\) and \(\mathcal M_1(D')\),
respectively. For any $y_1 \in \mathcal{Y}_1$, let $P_2^{y_1}$ and $Q_2^{y_1}$ be the distributions of $\mathcal{M}_2(D, y_1)$ and $\mathcal{M}_2(D', y_1)$, respectively.

Then, the distribution $P$ of $\mathcal{M}_D$ has marginal distribution $P_1$ on $\mathcal{Y_1}$ and $P_2^{y_1}$ on $\mathcal{Y_2}$ for any $y_1 \in \mathcal{Y}_1$, and the distribution $Q$ of $\mathcal{M}_D$ has marginal distribution $Q_1$ on $\mathcal{Y_1}$ and $Q_2^{y_1}$ on $\mathcal{Y_2}$ for any $y_1 \in \mathcal{Y}_1$

Since $\mathcal{M}_1$ is $(\mu_1, \delta_1)$-GDP, there exists a
probability distribution \(R_1\) on \(\mathcal Y_1\) such that
\[
T(P_1, R_1)\geq G_{\mu_1},
\qquad
\operatorname{TV}(Q_1, R_1) \leq \delta_1.
\]
Similarly, because \(\mathcal M_2(\cdot,y_1)\) is
\((\mu_2,\delta_2)\)-GDP for every \(y_1 \in \mathcal{Y}_1\), there exists, for each such \(y_1\), a probability distribution ${R}_2^{y_1}$ on $\mathcal{Y}_2$ such that
\[
T(P_2^{y_1}, R_2^{y_1})\geq G_{\mu_2},
\qquad
\operatorname{TV}(Q_2^{y_1}, R_2^{y_1}) \leq \delta_2.
\]
Define the distribution $R$ on $\mathcal{Y}_1 \times \mathcal{Y}_2$ that has marginal distribution $R_1$ of $\mathcal{Y_1}$ and $R_2^{y_1}$ over $\mathcal{Y_2}$ for any $y_1 \in \mathcal{Y}_1$.

The composition formula for GDP (equation~\eqref{eq:pure_gdp_composition}) implies that $T(P, R) \ge G_{\sqrt{\mu_1^2+\mu_2^2}}$. 

It remains to compare the distributions \(Q\) and \(R\). Take a maximal coupling $\pi_1$ of \(Q_1\) and \(R_1\). Moreover, for every $y_1 \in Y_1$, take a maximal coupling $\gamma_{y_1}$ of $Q_2^{y_1}$ and $R_2^{y_1}$.

Consider the following coupling $\Gamma$ between $Q$ and $R$:
\begin{enumerate}
    \item Sample $(y_1, y_1')$ from $\pi_1$.
    \item If $y_1 = y_1'$, then sample $(y_2, y_2')$ from $\gamma_{y_1}$. Otherwise, sample $(y_2, y_2')$ from the product distribution $Q_2^{y_1} \times R_2^{y_1'}$.
\end{enumerate}
Clearly, the distribution of $(y_1, y_2)$ is that of $Q$, and the distribution of $(y_1', y_2')$ is that of $R$. Moreover,
\bas{
\Pr((y_1, y_2) \neq (y_1', y_2')) &= \Pr(y_1 \neq y_1') + \Pr(y_1 = y_1' \text{ and } y_2 \neq y_2'). \\
&= \Pr(y_1 \neq y_1') + \Pr(y_2 \neq y_2' | y_1 = y_1')\Pr(y_1 = y_1') \\
&\le \Pr(y_1 \neq y_1') + \delta_2\Pr(y_1 = y_1') \\
&= (1-\delta_2)\Pr(y_1 \neq y_1') + \delta_2 \\
&\le \delta_1(1-\delta_2) + \delta_2 = \delta_1+\delta_2 - \delta_1 \delta_2.
}
Consequently, $TV(Q, R) \le \delta_1+\delta_2 - \delta_1 \delta_2.$ We have constructed an auxiliary distribution $R$ satisfying
\[
T(P,R) \geq G_{\sqrt{\mu_1^2+\mu_2^2}},
\qquad
\operatorname{TV}(Q,R)\leq \delta_1+\delta_2 - \delta_1 \delta_2.
\]
Since $P$ and $Q$ are the distributions of $\mathcal{M}(D)$ and $\mathcal{M}(D)'$ respectively, the result follows.
\end{proof}

\subsection{Proof of Lemma~\ref{lem:ptr_gdp_privacy}}

\begin{proof}
Fix neighboring datasets \(D\) and \(D'\). Let $\theta = \theta(D)$ and $\theta' = \theta(D').$ 

By the definition of local sensitivity, $\norm{\theta-\theta'}_2 \le LS(D'),$ where the local sensitivity is with respect to the statistic $\theta(\cdot)$. 
Define $L' = LS(D')$ and $\Delta = \norm{\theta-\theta'}_2$. Then, $\Delta \le L'$

Let $P_1$ and $Q_1$ be the distributions of $\mathcal B(D)$ and $\mathcal B(D')$
respectively, which is supported on $\R_{> 0}$. For every positive real number $b$, define the normal distributions
\[
P_2^{(b)} = \mathcal{N}\bb{\theta, \frac{b^2}{\mu_2^2} \id_d}, \qquad Q_2^{(b)} =  \mathcal{N}\bb{\theta', \frac{b^2}{\mu_2^2} \id_d}.
\]
Then, the distribution $P$ on $\R \times \R^d$ of the tuple $\bb{\mathcal{B}(D), \theta(D) + \frac{\mathcal{B}(D)}{\mu_2} \cdot Z}$, where $Z$ is the $d$-dimensional standard normal distribution, can be described by sampling a positive real $b$ from $\mathcal{B}(D)$ and then sampling a vector $\mathbf{v} \in \R^d$ from $P_2^{(b)}$. Define the distribution $Q$ similarly.


To compare the laws $P$ and $Q$, we introduce an auxiliary distribution $R$ that will be shown to be close to both. The distribution $R$ supported on $\R_{>0}\times\R^d$ is defined through the following sampling procedure. First, draw
$B \sim Q_1.$ Conditional on $B=b$, draw $\mathbf V$ according to
\ba{
(\mathbf V\mid B=b) \sim
\begin{cases}
P_2^{(b)}, & b<L',\\
Q_2^{(b)}, & b\geq L'.
\end{cases} \label{eq:R_second_marginal}
}
The joint law of $(B,\mathbf V)$ is denoted by $R$.

First, $Q$ and $R$ are close in total variation. Indeed, consider any measurable set $E \subseteq \R_{> 0} \times \R^d.$ For each $b > 0$, let $E_b$ be the set of vectors $\mathbf{v} \in \R^d$ such that $(b, \mathbf{v}) \in E$. Then,
\bas{
|Q(E)-R(E)| &= \left | \int_{b > 0} \bb{Q_2^{(b)}(E_b) - R_2^{(b)}(E_b)} Q_1(b) \,\mathrm{d}b \right | \\
&= \left | \int_{b < L'} \bb{Q_2^{(b)}(E_b) - P_2^{(b)}(E_b)} Q_1(b) \, \mathrm{d}b \right | \\
&\le \int_{b < L'} \left | Q_2^{(b)}(E_b) - P_2^{(b)}(E_b) \right | Q_1(b) \,\mathrm{d}b \\
&\le \int_{b < L'} Q_1(b) \,\mathrm{d}b = \Pr(\mathcal{B}(D') < L'),
}
which is at most $\delta$ by the definition of $\mathcal{B}$. Taking a supremum over all such $E$, it follows that $\TV(Q, R) \le \delta$.

Next, we show that $T(P,R)\ge G_{\sqrt{\mu_1^2+\mu_2^2}}$. The first-marginals of $P$ and $R$ are $P_1$ and $Q_1$. Since $\mathcal B$ is $\mu_1$-GDP, $T(P_1, Q_1)\ge G_{\mu_1}.$

Conditioned on the first marginal $b > 0$, the conditional law of the
second marginal of $P$ is $P_2^{(b)}$ and that of $R$ is $R_2^{(b)}$, which is equal to $P_2^{(b)}$ if $b < L'$ and $Q_2^{(b)}$ if $b \ge L'$

If $b < L'$, then the conditional distributions $P_2^{(b)}$ and $R_2^{(b)}$ are identical. 

If $b \ge L'$, then $\Delta\le L'\le b.$ Consider the normal distributions 
\[
P_2^{(b)} = \mathcal{N}\!\left(\theta,\frac{b^2}{\mu_2^2}\id_d\right)
\quad\text{versus}\quad
Q_2^{(b)} = \mathcal{N}\!\left(\theta',\frac{b^2}{\mu_2^2}\id_d\right).
\]
The trade-off curve between these two distributions is $G_{\mu'}$ for $\mu' = \frac{\Delta}{b/\mu_2} \le \mu_2.$

Either way, the conditional second coordinate is $\mu_2$-GDP for every
$b$. By composition of GDP mechanisms (equation~\eqref{eq:pure_gdp_composition}), $T(P,R)\ge
G_{\sqrt{\mu_1^2+\mu_2^2}}.$ In all, the distribution $R$ satisfies
\[
T(P,R)\ge G_{\sqrt{\mu_1^2+\mu_2^2}}
\qquad\text{and}\qquad
\operatorname{TV}(Q,R)\le\delta.
\]
Therefore, $T(P,Q)(\alpha)
\ge
G_{\sqrt{\mu_1^2+\mu_2^2},\delta}(\alpha),$ and the output 
\[
\bb{\mathcal{B}(D), \theta(D) + \frac{\mathcal{B}(D)}{\mu_2} \cdot \mathcal{N}(0, \id_d)}
\]
is $(\sqrt{\mu_1^2 + \mu_2^2}, \delta)$-GDP.

Finally, consider \texttt{PTR-GDP} (Algorithm~\ref{alg:ptr_gdp}), which defines
\[\mathcal{B}(D) \gets \SS_{\beta}(D) \cdot \exp\bb{\frac{\beta}{\mu} \cdot Z_1 + \frac{\beta}{\mu} \cdot \Phi^{-1}(1-\delta)}.\]

Note that $\mathcal{B}(D)$ as defined above is $\mu$-GDP. Indeed, since $\SS_{\beta}$ is $\beta$-smooth, $\log \SS_{\beta}$ has global sensitivity at most $\beta$. It follows that $\log \SS_{\beta} + \frac{\beta}{\mu} \cdot Z_1$ is $\mu$-GDP, so $\mathcal{B}(D)$ is also $\mu$-GDP by post-processing.

By the definition of smooth sensitivity, $\SS_{\beta}(D) \ge LS(\theta, D)$. Also, $Z_1 \ge \Phi^{-1}(1-\delta)$ with probability equal to $1-\delta.$ Therefore, $\mathcal{B}(D) \ge LS(\theta, D)$ with probability at least $1-\delta$. In all, it follows that \texttt{PTR-GDP} is approximately GDP with parameters $(\sqrt{\mu^2 + \mu^2}, \delta) = (\sqrt{2} \mu, \delta)$.
\end{proof}

\section{Proofs related to \dpss}
\label{app:alg4}

\subsection{Proof of Lemma~\ref{lem:smooth_each_subsample}}

\begin{proof}
Consider two neighboring datasets $D$ and $D'$. Since any index appears at most $\kstar$ times in $S_i$, the subsampled datasets $D[S_i]$ and $D'[S_i]$ differ in at most $\kstar$ elements. Consider a sequence of adjacent datasets $D_0 = D[S_i] \sim D_1 \sim D_2, \dots, \sim D_r = D'[S_i],$ where $r \le \kstar$. 

Let $d(x,y) = \norm{x-y}_2$. By the triangle inequality,
\[
d(\theta(D[S_i]),\theta(D'[S_i])) = d(\theta(D_0),\theta(D_r)) \le \sum_{q=0}^{r-1} d(\theta(D_q),\theta(D_{q+1})) \le \sum_{q=0}^{r-1} \LS(\theta, D_q).
\]
Now, $\LS(\theta, D_0) \le \SS_{\beta}(\theta, D[S_i]) =: s_i$ when $D[S_i]$ is viewed as a dataset on $m$ ``distinct'' elements. Moreover, since $\SS_{\beta}$ is $\beta$-smooth, $\LS(\theta, D_1) \le e^{\beta}s_i, \LS(\theta, D_2) \le e^{2\beta}s_i,$ and in general $\LS(\theta, D_q) \le e^{q\beta}s_i$ for all $0 \le q \le r-1$. Thus,
\[
d(\theta(D[S_i]),\theta(D'[S_i])) \le \sum_{q=0}^{r-1} e^{q\beta}s_i \le \frac{e^{\kstar\beta}-1}{e^{\beta}-1} s_i.
\]
Finally, note that the local sensitivity of the function
\[
\texttt{Quantile}_{1-\alpha}\bb{d(\theta_1^{\star}, \hatn), d(\theta_2^{\star}, \hatn), \dots, d(\theta_B^{\star}, \hatn)}
\]
is bounded by the maximum of the local sensitivities of $d(\theta_i^{\star}, \hatn)$, which is further upper bounded by the sum of the local sensitivities of $\theta_i^{\star}$ and $\hatn$ due to triangle inequality. Therefore,
\[
\LS(\texttt{Quantile}_{1-\alpha}\bb{d(\theta_1^{\star}, \hatn), d(\theta_2^{\star}, \hatn), \dots, d(\theta_B^{\star}, \hatn)}) \le \max_{i \in [B]} \SS_{\beta}(D[S_i]) \cdot \frac{e^{\kstar \beta}-1}{e^{\beta}-1} + \SS_{\beta}(D) = s(D),
\]
and $\LS(\qnm(D)) \le \tau_m s(D).$
\end{proof}

\subsection{Proof of privacy of Theorem~\ref{thm:smooth_utility}}

\begin{proof}
The resampling indices are drawn independently of the data.  It is
therefore enough to prove privacy for every fixed realization of the
resampling indices satisfying the multiplicity constraint.  For a
realization that does not satisfy the constraint, the algorithm returns
the fixed symbol $\bot$ and is trivially private. Fix an (accepted) realization $(S_1,\ldots,S_B)$, and define $r_{\beta,\kappa_\star}
:=
\frac{e^{\kappa_\star\beta}-1}{e^\beta-1}.$
For neighboring datasets $D\sim D'$, the datasets $D[S_i]$ and
$D'[S_i]$ differ in at most $\kappa_\star$ entries.  The telescoping
argument in Lemma~\ref{lem:smooth_each_subsample} therefore gives
\[
d\bigl(
    \theta(D[S_i]),
    \theta(D'[S_i])
\bigr)
\leq
r_{\beta,\kappa_\star}
SS_\beta(\theta,D[S_i]).
\]
Moreover, $d\bigl(\theta(D),\theta(D')\bigr)
\leq SS_\beta(\theta,D).$
Consequently, if
\[
a_i(D)
:=
d\bigl(\theta(D[S_i]),\theta(D)\bigr),
\]
then $|a_i(D)-a_i(D')|
\leq
r_{\beta,\kappa_\star}
SS_\beta(\theta,D[S_i])
+
SS_\beta(\theta,D).$
Since an empirical quantile, viewed as an order statistic, is
$1$-Lipschitz with respect to the $\ell_\infty$ norm of its inputs,
\[
LS(Q_{n,m};D)
\leq
\tau_m
\left(
    r_{\beta,\kappa_\star}
    \max_{i\in[B]}SS_\beta(\theta,D[S_i])
    +
    SS_\beta(\theta,D)
\right)
=
\tau_m s(D).
\]

It remains to verify the smoothness of this upper bound.  Note that
$D[S_i]$ and $D'[S_i]$ differ in at most $\kappa_\star$ entries, so
\[
SS_\beta(\theta,D[S_i])
\leq
e^{\kappa_\star\beta}
SS_\beta(\theta,D'[S_i]).
\]
Also,
\[
SS_\beta(\theta,D)
\leq
e^\beta SS_\beta(\theta,D')
\leq
e^{\kappa_\star\beta}SS_\beta(\theta,D').
\]
It follows that $s(D)\leq e^{\kappa_\star\beta}s(D'),$
and the reverse inequality follows by interchanging $D$ and $D'$.
Thus, $\tau_m s$ is a $\kappa_\star\beta$-smooth upper bound on
$LS(Q_{n,m};\cdot)$.

The first call to PTR-GDP uses the $\beta$-smooth upper bound
$SS_\beta(\theta,\cdot)$ and privacy parameters $(\mu/2,\delta/2)$.
By Lemma~\ref{lem:ptr_gdp_privacy} and post-processing, its output $\widetilde\theta_n$ is
$\left({\mu}/{\sqrt{2}},{\delta}/{2}\right)
\text{-GDP}.$ The second call uses the $\kappa_\star\beta$-smooth upper bound
$\tau_m s$ and the same privacy parameters.  Hence
$\widetilde Q_{n,m}$ is also $\left({\mu}/{\sqrt{2}},{\delta}/{2}\right)
\text{-GDP}.$

The final projection onto $[0,\infty)$ is post-processing. By Proposition~5, the joint output is
\[
\left(
    \sqrt{
        \frac{\mu^2}{2}+\frac{\mu^2}{2}
    },
    1-\left(1-\frac{\delta}{2}\right)^2
\right)
=
\left(
    \mu,
    \delta-\frac{\delta^2}{4}
\right)
\text{-GDP}.
\]
Since $\delta-\delta^2/4\leq\delta$, the conditional mechanism is
$(\mu,\delta)$-GDP. Since the resampling indices and the accept/reject decision are independent of the data, releasing them first costs no privacy. For each realization, the subsequent mechanism is $(\mu, \delta)$-GDP. Dropping the indices is post-processing.
\end{proof}

\subsection{Proof of utility of Theorem~\ref{thm:smooth_utility}}

As in the proof of Theorem~1, regard all resamples as drawn before the
multiplicity check and define the auxiliary outputs
\[
\bar\theta_n
=
\widehat\theta_n+\frac{2\mathcal B_n}{\mu}G_1,
\qquad
\bar Q_{n,m}
=
\max\left\{
Q_{n,m}+\frac{2\mathcal C_{n,m}}{\mu}G_2,0
\right\},
\]
where \(G_1\sim N(0,I_d)\), \(G_2\sim N(0,1)\), and
\[
\mathcal B_n
=
SS_\beta(\theta,D)
\exp\left\{
\frac{2\beta}{\mu}
\left(Z_1+\Phi^{-1}(1-\delta/2)\right)
\right\},
\]
\[
\mathcal C_{n,m}
=
\tau_m s(D)
\exp\left\{
\frac{2\kappa_\star\beta}{\mu}
\left(Z_2+\Phi^{-1}(1-\delta/2)\right)
\right\}.
\]
Here \(Z_1,Z_2\sim N(0,1)\) are independent of \(G_1,G_2\).

By the assumptions on the smooth sensitivities, the choice of
\(\beta\), and a union bound, with probability at least \(1-O(\rho)\),
\[
\mathcal B_n\leq e\Gamma_{n,\rho}
\]
and
\[
\mathcal C_{n,m}
\leq
\tau_m
\left(
\kappa_\star e^2\Gamma_{m,\rho/B}
+
e\Gamma_{n,\rho}
\right).
\]
Indeed,
\[
\frac{e^{\kappa_\star\beta}-1}{e^\beta-1}
=
\sum_{j=0}^{\kappa_\star-1}e^{j\beta}
\leq
\kappa_\star e^{\kappa_\star\beta}
\leq
\kappa_\star e.
\]

Let
\[
T_n:=\tau_n d(\theta,\widehat\theta_n),
\qquad
\bar T_n:=\tau_n d(\theta,\bar\theta_n).
\]
Since projection onto \([0,\infty)\) is \(1\)-Lipschitz,
\[
|\bar T_n-T_n|
\leq
\frac{2\tau_n\mathcal B_n}{\mu}\|G_1\|_2,
\qquad
|\bar Q_{n,m}-Q_{n,m}|
\leq
\frac{2\mathcal C_{n,m}}{\mu}|G_2|.
\]
Gaussian concentration and the assumptions
\[
\Gamma_{n,\rho}
=
o\left(
\frac{\mu}{
\tau_n(\sqrt d+\sqrt{\log(1/\rho)})}
\right),
\qquad
\Gamma_{m,\rho/B}
=
o\left(
\frac{\mu}{
\tau_m\kappa_\star\sqrt{\log(1/\rho)}}
\right)
\]
therefore imply
\[
\bar T_n-T_n=o_p(1),
\qquad
\bar Q_{n,m}-Q_{n,m}=o_p(1).
\]
By Assumption~\ref{ass:limit},
\[
T_n\xrightarrow{d}T,
\qquad
Q_{n,m}\xrightarrow{p}q_{1-\alpha},
\quad
\text{where } q_{1-\alpha}:=F^{-1}(1-\alpha).
\]
By Slutsky's theorem,
\[
\bar T_n\xrightarrow{d}T,
\qquad
\bar Q_{n,m}\xrightarrow{p}q_{1-\alpha}.
\]
Lemma~\ref{lem:coverage_estimated_quantile} now gives
\[
\Pr(\bar T_n<\bar Q_{n,m})
=
1-\alpha+o_n(1).
\]
The auxiliary outputs agree with the outputs of Algorithm~\ref{alg:smoothed_quantile_subsample} whenever the multiplicity check succeeds, whose failure probability is at most
\(\rho=o_n(1)\). Therefore,
\[
\Pr\left(
d(\theta,\widetilde\theta_n)
<
\frac{\widetilde Q_{n,m}}{\tau_n}
\right)
=
1-\alpha+o_n(1).
\]

Finally,
\[
\frac{\widetilde Q_{n,m}}{\tau_n}
=
\max\left\{
\frac{Q_{n,m}}{\tau_n}
+
\sigma_{\mathrm{priv}}G_2,0
\right\},
\qquad
\sigma_{\mathrm{priv}}
=
\frac{2\mathcal C_{n,m}}{\mu\tau_n},
\]
where \(G_2\) is independent of \(\sigma_{\mathrm{priv}}\). Thus, with
probability at least \(1-O(\rho)\),
\[
\sigma_{\mathrm{priv}}
\leq
\frac{2\tau_m}{\mu\tau_n}
\left(
\kappa_\star e^2\Gamma_{m,\rho/B}
+
e\Gamma_{n,\rho}
\right).
\]
Moreover,
\[
\frac{2\mathcal C_{n,m}}{\mu}G_2=o_p(1).
\]
Since \(Q_{n,m}\xrightarrow{p}q_{1-\alpha}>0\), the truncation is
inactive with probability \(1-o_n(1)\).

\section{Proofs for Confidence Intervals of Quantiles (Section~\ref{sec:quantile})}
\label{app:quantiles}

\begin{theorem}
\label{thm:ss_density_quantile_bound}
Let $\rho,\zeta, {\zeta'} \in (0,1)$ with $\zeta'+\frac{2}{n} \le \zeta \le 1-\zeta'-\frac{2}{n}$, and suppose $\mathcal{D}$ is a distribution supported on $[\ell, u]$ such that the density $f$ of $\mathcal{D}$ is continuous and satisfies
\bas{
    \min_{x} \left \{f(x) : F(x) \in \left[1-\zeta-{\zeta'}-\sqrt{{(\log 1/\rho)}/{2n}},1-\zeta+{\zeta'}+\sqrt{{(\log 1/\rho)}/{2n}}\right] \right \} \geq \lambda.
}
Let $m = \lfloor n(1-\zeta) \rfloor$ and $k = \lceil n\zeta' \rceil$, and assume $\zeta'$ satisfies
\[
\max\bb{\frac{2}{n} + \frac{1}{n\beta}\Big(\log\!\frac{u-\ell}{C}+3\log n\Big), \frac{8\log (1/\rho)}{n}} \le \zeta' \le \frac{1}{2} - \frac{1+8\log(1/\rho)}{2n}.
\]
Then, with probability at least $1-4\rho$, the $\beta$-smooth function $\SS_{\beta}^{(1-\zeta)}(\theta, D) \le \frac{6k}{n \lambda} + \frac{C}{n^3}.$
\end{theorem}

\subsection{Proof of Lemma~\ref{lem:ss_quantile}}
The expression for $\SS_{\beta}^{(1-\zeta)}(\theta, D)$ is the smooth sensitivity
construction of~\cite{nissim2007smooth} for the median and the minimum generalized to quantiles. For a dataset $D'$ which has Hamming distance at most $k$ from $D$, the index of the $(1-\zeta)$-quantile can shift by at most $k$ positions, and the worst-case shift of the order statistic within this window is captured by $\max_{t=0}^{k+1}(x_{(h+t)}-x_{(h+t-k-1)})$. This yields a
$\beta$-smooth upper bound on the local sensitivity.

For the truncation, note that for any $k$ and $t$ we always have
\[
0 \le x_{(h+t)}-x_{(h+t-k-1)} \le u-\ell,
\]
by the extension $x_{(j)}\in[\ell,u]$. Hence, for $k\ge k^\star$,
\[
e^{-k\beta}\max_{0\le t\le k+1}\Big(x_{(h+t)}-x_{(h+t-k-1)}\Big)
\le e^{-k^\star \beta}(u-\ell) \le \frac{C}{n^3},
\]
where the last inequality is exactly the definition of $k^\star$.
Therefore,
\[
\SS_{\beta}^{(1-\zeta)}(\theta, D)
\le
\max_{0\le k\le k^\star} e^{-k\beta}\,
\max_{0\le t\le k+1}\Big(x_{(h+t)}-x_{(h+t-k-1)}\Big)
\;+\;\frac{C}{n^3}.
\]

Finally, since $e^{-k\beta}\le 1$ and enlarging $k$ enlarges the index range,
\[
\max_{0\le k\le k^\star}\max_{0\le t\le k+1}\Big(x_{(h+t)}-x_{(h+t-k-1)}\Big)
\le
\max_{0\le t\le k^\star+1}\Big(x_{(h+t)}-x_{(h+t-k^\star-1)}\Big)
\le
x_{(h+k^\star+1)}-x_{(h-k^\star-1)},
\]
which gives the stated simplified bound.

\subsection{Proof of Theorem~\ref{thm:ss_density_quantile_bound}}

\begin{lemma}
    \label{lem:quantilespacing}
    Let $\rho \in (0,1)$, and suppose $\zeta, {\zeta'} \in (0,1)$ satisfy 
    \[
    \zeta'+\frac{2}{n} \le \zeta \le 1-\zeta'-\frac{2}{n}, \qquad 
    \frac{8\log (1/\rho)}{n} \le {\zeta'} \le \frac{n-1-8\log(1/\rho)}{2n}.
    \]
    Suppose $\mathcal{D}$ is a distribution supported on $[\ell, u]$ and let $X_1,\dots,X_n$ be i.i.d draws from $\mathcal{D}$. Let $X_{(1)}\leq X_{(2)}\leq \dots X_{(n)}$ denote the order statistics and let $h=\lfloor n(1-\zeta)\rfloor$. 
    Define a neighborhood of the $1-\zeta$ quantile $F^{-1}(1-\zeta)$ as the set $W=\{x:|F(x)-(1-\zeta)| \le {\zeta'}+\sqrt{(\log 1/\rho)/2n}\}$.
    If the density $f$ of $\mathcal{D}$ is continuous and satisfies
    \bas{
    \min_{x\in W} f(x)\geq \lambda,
    }
    then, for $k=\lceil n{\zeta'}\rceil$,
    \bas{
    \Pr\bb{X_{(h+k)}-X_{(h-k)} \geq \frac{6k}{n \lambda}}\leq 4\rho.
    }
\end{lemma}

\begin{proof}
    With the constraint on $\zeta$ and $\zeta'$, we know that $1 \le h-k \le h+k \le n$. Define $U_j:=F(X_j)$ and note that $U_j$ is distributed as $\operatorname{Uniform}([0,1])$. By mean value theorem,
    \bas{
    X_{(h+k)}-X_{(h-k)}&=F^{-1}(U_{(h+k)})-F^{-1}(U_{(h-k)})=\frac{U_{(h+k)}-U_{(h-k)}}{f(\xi)}
    }
    for some $\xi \in [F^{-1}(U_{(h-k)}), F^{-1}(U_{(h+k)})]$. 

    For any index $j$ and $t$, $\Pr[U_{(j)} \ge t]$ is equal to $\Pr[\operatorname{Bin}(n,t) \le j-1]$ (the probability that at most $j-1$ of the $n$ independent samples from $\operatorname{Uniform}([0,1])$ exceed $j-1$). Since $\E[\operatorname{Bin}(n,t)] = nt,$ applying a Chernoff bound with $j = h+k$ and $t = \frac{h+k}{n} + \eps$ for $\eps = \sqrt{\frac{\log 1/\rho}{2n}}$ gives
    \[
    \Pr[U_{(h+k)} \ge \frac{h+k}{n} + \eps] \le e^{-2n\eps^2} = \rho
    \]
    A similar lower tail bound for $U_{(h-k)}$ implies
    \[
    \Pr[U_{(h-k)} \le \frac{h-k}{n} - \eps] \le e^{-2n\eps^2} = \rho
    \]
    Consequently, with probability at least $1-2\rho$ the value $F(\xi) \in [1-\zeta-{\zeta'}-\sqrt{\frac{\log 1/\rho}{2n}}, 1-\zeta+{\zeta'}+\sqrt{\frac{\log 1/\rho}{2n}}]$.
    
    Under this event and by assumption,
    \bas{
     X_{(h+k)}-X_{(h-k)}&\leq \frac{U_{(h+k)}-U_{(h-k)}}{\lambda}.
    }

    Now, define the random variables $\{G_i\}_{0 \le i \le n}$ as $G_0=U_{(1)},$ $G_i = U_{(i+1)}-U_{(i)}$ for $i\in [1,n-1]$, and $G_n = 1-U_{(n)}$. By the R{\'e}nyi representation of order statistics~\citep{renyi1953theory, david2004order},
    \ba{(G_0, G_1,\dots,G_n)\stackrel{d}{=} \bb{\frac{Y_0}{S},\dots, \frac{Y_n}{S}} \label{eq:uniform_renyi_representation}
    }
    where $Y_i \sim \operatorname{Exp}(1)$ i.i.d for $0\leq i\leq n$ and let $S=\sum_i Y_i$.
    Therefore
    \[
    U_{(h+k)}-U_{(h-k)} = \frac{\sum_{i=h-k}^{h+k-1} G_i}{\sum_{i=0}^{n} G_i} = \frac{A}{A+B},
    \]
    where $A \sim \Gamma(2k,1)$ and $B \sim \Gamma(n-2k+1, 1)$ are independent. By a {Chernoff bound},
    \[
    \Pr[A > 3k] \le e^{-k/8}, \qquad \Pr[B < (n-2k+1)/2] \le e^{-(n-2k+1)/8}.
    \]

    By assumption on ${\zeta'}$, both $k/8$ and $(n-2k+1)/8$ are at least $\log 1/\rho$. Overall, with probability at least $1-4\rho$,
    \[
    U_{(h+k)}-U_{(h-k)} \le \frac{3k}{3k+(n-2k+1)/2} = \frac{6k}{n+4k+1} < \frac{6k}{n}.
    \]
    
Since Theorem~\ref{thm:ss_density_quantile_bound} has the assumptions of Lemma~\ref{lem:quantilespacing}, we have with probability at least $1-4\rho$ that
\[
X_{(h+k)}-X_{(h-k)} \le \frac{6k}{n \lambda}.
\]
\end{proof}

We also have
\[
k-1 \ge n\zeta'-1 \ge 1 + \frac{1}{\beta}\Big(\log\!\frac{u-\ell}{C}+3\log n\Big) \ge \left \lceil \frac{1}{\beta}\Big(\log\!\frac{u-\ell}{C}+3\log n\Big) \right \rceil 
\]
so Lemma~\ref{lem:ss_quantile} implies that
\[
\SS_{\beta}^{(1-\zeta)}(\theta, D) \le X_{(h+k)}-X_{(h-k)} + \frac{C}{n^3} \le \frac{6k}{n \lambda} + \frac{C}{n^3}.
\]

\subsection{Calculating the confidence radius}

We now apply Theorem~\ref{thm:smooth_utility} with $\Gamma_{N,\rho} := \frac{6K_{N}}{N\lambda} + \frac{C}{N^3}$ for all $N,\rho$ as the upper bound on the smooth sensitivity above (the failure probability $\rho$ is pushed into the condition on the density around the $1-\zeta$ quantile) and
\[
K_{N} := \max\{\beta^{-1} (\log (C^{-1}(u-\ell)) +3\log N), 0\}.
\]
We use Theorem~\ref{thm:ss_density_quantile_bound} to obtain a utility guarantee for confidence intervals for quantiles. As in the theorem, let
\[
\beta = \frac{1}{2\kstar} \min\bb{1, \frac{\mu}{4 \sqrt{\log B/(\delta \rho)}}}.
\]
Suppose \[f(x) \ge \lambda \text{ for all } F(x) \in 1-\zeta \pm \bb{\zeta' + \sqrt{\max\bb{\frac{\log (B/\rho)}{2m}, \frac{\log (1/\rho)}{2n}}}}.\]

If Assumption~\ref{ass:limit} holds, $\rho = o_n(1),$ and 
\[
\frac{(\log n)\bb{\sqrt{d} + \sqrt{\log(1/\rho)}}}{\sqrt{n} \lambda \mu} + \frac{(\log m)\sqrt{\log(1/\rho)}}{\sqrt{m} \lambda \mu} = o_n(1),
\]

then with probability at least $1-8\rho$, \texttt{CR-SS} (Algorithm~\ref{alg:smoothed_quantile_subsample}) outputs an asymptotically $1-\alpha$ confidence interval $\tiln \pm \qnmp/\tau_n$ for the $1-\zeta$ quantile of the distribution $\mathcal{D}$. Moreover, the half-length of this confidence interval is equal to the non-private radius $\qnm/\tau_n$ plus the private noise $\sigma_{\text{priv}} G$, where $G$ is a standard normal with scale $\sigma_{\text{priv}}$ given by
\bas{
\sigma_{\text{priv}} &\le \frac{2\tau_m}{\mu \tau_n}\bb{\Gamma_{m,\rho/B}\kstar e^2 + \Gamma_{n,\rho}e} = O\bb{\frac{\kstar k}{\mu \sqrt{mn}\lambda } + \frac{m^{1/2}}{n^{3/2}\lambda } + \frac{C \kstar}{m^{5/2} n^{1/2}}} = \tilde{O}\bb{\frac{1}{\sqrt{mn}}}.
}

Note that the non-private error of a quantile when using $n$ samples is $O(1/\sqrt{n})$ whereas the private error we obtain has scale $\tilde{O}(1/\sqrt{mn})$.



\section{Proofs for Degenerate U-statistics (Section~\ref{sec:degenerate})}
\label{app:degenerate}

\paragraph{Private estimation for non-degenerate and degenerate $U$-statistics}
~\cite{chaudhuri2024differentially}  provide a $\beta$-smooth upper bound on $\LS(\widehat{U}_n)$ that satisfies 
\[
\SS_{\beta}(\widehat{U}_n) \le \frac{r}{n} \bb{\xi + \frac{rC(1/\beta  + L)}{n}}\bb{1+\beta(1+ L)} + \frac{r^2}{n^2} C \bb{\frac{2}{\beta} + L}^2 \min\bb{r  , \frac{2}{\beta} + L}\bb{\beta+\frac{r}{n}}.
\]
where $C$ is the diameter of the support of the distribution, $L = L(D)$ is a positive integer which is an upper bound on the number of \textit{bad} data points in $D$ (we omit the exact details on what a bad data point is here) and $\xi$ is a concentration parameter. This bound works for any statistic, not necessarily degenerate U-statistics.

The parameter $\xi$ is such that the $n$ U statistics $A_1, \dots, A_n$, which they call Hájek projections, each satisfy $|A_k - U_n| \le \xi ~\forall~k$ with high probability. The Hâjek projections $A_i$ are defined as
\[
A_k := \frac{1}{\binom{n-1}{r-1}} \sum_{1 \le i_1 < i_2 < \dots < i_{r-1} \le n, i_j \neq k \forall j} h(X_k, X_{i_1}, \dots, X_{i_{r-1}}),
\]
which fixes $X_k$ to be one of the inputs of $h$ and sums over other $r-1$ out of $n-1$ possible inputs.

In general, $\xi$ can always be taken as $C$, which is the range of the kernel $h$. For degenerate $U$-statistics, Lemma~6~of~\cite{chaudhuri2024differentially} shows that $\xi$ can be taken as $C\sqrt{{r}/{n} \log (2n/\rho)} + {8Cr}/{3n} \log(2n/\rho).$

Therefore, we can apply Theorem~\ref{thm:smooth_utility} with $\beta$ as in the theorem and
\bas{
\Gamma_{n,\rho} &:= \frac{r}{n} \bb{\xi + \frac{rC(1/\beta  + L)}{n}}\bb{1+2\beta(1+ L)} + \frac{r^2}{n^2} C \bb{\frac{2}{\beta} + L}^2 \min\bb{r  , \frac{2}{\beta} + L}\bb{\beta+\frac{r}{n}}.
}

Therefore, we can apply Theorem~\ref{thm:smooth_utility} with $\beta$ as in the Theorem and $\tau_N = N$ (assuming the statistic is degenerate of order $1$). Another point to note is that with the beforementioned choice of $\xi$, $L(D) = 1$ with probability at least $1-\rho$ (where the randomness is over the data distribution, not the algorithm). In this case, the smooth sensitivity at a \textit{typical} input is is bounded above as
\[
\Gamma_{n,\rho} \le \frac{r^{3/2} C \log(2n/\rho) (1+2\beta)}{n^{3/2}} + t_{n,\beta},
\]
where $t_{n,\beta} = O_{\beta}(1/n^2)$ for a fixed $r$. It follows that with probability $1-O(\rho)$ the private overhead of the confidence radius is equal to $\sigma_{\text{priv}} \cdot G$ where
\[
\sigma_{\text{priv}} \le \frac{2m}{n\mu}(\Gamma_{m,\rho/B}\kstar e^2 + \Gamma_{n,\rho} e) = \tilde{O}\bb{\frac{r^{3/2}\log(Bn)}{n\sqrt{m}}}.
\]

\input{CRGS-SGD}

%% file: map.tex
The appendix is organized as follows.
\begin{enumerate}
    \item Section~\ref{app:auxiliary} contains some known auxiliary results.
    \item Section~\ref{sec:resampDetails} contains details of subsampling validity.
    \item Section~\ref{sec:appendix:compare} has detailed calculations that led to Table~\ref{tab:rw_comparison}.
    \item Section~\ref{app:global} has the privacy and utility guarantees for \texttt{DP-QuantS}.
    \item Section~\ref{app:ptr_gdp_proof} proves the privacy of \texttt{PTR-GDP}.
    \item Section~\ref{app:alg4} has the privacy and utility guarantees for \texttt{Quant-SS}.
    \item Section~\ref{app:quantiles} has the guarantees for confidence intervals for quantiles.
    \item Section~\ref{app:degenerate} has the guarantees for degenerate U-statistics.
    
\end{enumerate}

%% file: comparison.tex

\begin{lemma}[~\label{lem:quantile_finiteB}Finite-\(B\) error of the empirical resampling quantile]
Let
\[
    Y_i=\tau_m d(\widehat\theta_i^\ast,\widehat\theta_n),
    \qquad i=1,\ldots,B,
\]
denote the scaled resampling pivots, conditional on the data \(D\), and let
\(F_{n,m}\) be their conditional CDF. Set \(p=1-\alpha\),
\[
    q_{1-\alpha}:=F_{n,m}^{-1}(p),
    \qquad
    Q_{n,m}:=Y_{(k)},
    \qquad
    k=\lceil Bp\rceil .
\]
Assume that \(F_{n,m}^{-1}\) is twice continuously differentiable in a
neighborhood of \(p\), with derivatives bounded independently of \(B\). Then,
conditional on \(D\),
\[
    \mathbb E[Q_{n,m}^2\mid D]
    =
    q_{1-\alpha}^2+O(B^{-1}).
\]
Consequently, the non-private squared length satisfies
\[
    \mathbb E[L_{\rm np}^2\mid D]
    =
    \frac{4q_{1-\alpha}^2}{\tau_n^2}
    +
    O\!\left(\frac{1}{\tau_n^2B}\right).
\]
\end{lemma}

\begin{proof}
Let \(U_i=F_{n,m}(Y_i)\). Conditional on \(D\),
\(U_1,\ldots,U_B\) are i.i.d. \({\rm Unif}(0,1)\), and
\[
    F_{n,m}(Q_{n,m})=U_{(k)}.
\]
By the standard order-statistic representation
\citep[Example~2.3]{david2004order},
\[
    U_{(k)}\sim {\rm Beta}(k,B-k+1).
\]
Hence
\[
    \mathbb E[U_{(k)}\mid D]
    =
    \frac{k}{B+1}
    =
    p+O(B^{-1}),
\]
and
\[
    \operatorname{Var}(U_{(k)}\mid D)
    =
    \frac{k(B-k+1)}{(B+1)^2(B+2)}
    =
    O(B^{-1}).
\]
Therefore
\[
    \mathbb E[(U_{(k)}-p)^2\mid D]=O(B^{-1}).
\]
Now define
\[
    h(u):=\{F_{n,m}^{-1}(u)\}^2 .
\]
Since \(h\) is twice continuously differentiable near \(p\), Taylor expansion
around \(p\) gives
\[
    h(U_{(k)})
    =
    h(p)
    +
    h'(p)(U_{(k)}-p)
    +
    O\!\left((U_{(k)}-p)^2\right).
\]
Taking conditional expectations,
\[
    \mathbb E[Q_{n,m}^2\mid D]
    =
    h(p)
    +
    h'(p)\,\mathbb E[U_{(k)}-p\mid D]
    +
    O\!\left(\mathbb E[(U_{(k)}-p)^2\mid D]\right).
\]
Using the bounds above,
\[
    \mathbb E[Q_{n,m}^2\mid D]
    =
    h(p)+O(B^{-1})
    =
    q_{1-\alpha}^2+O(B^{-1}).
\]
Finally, since \(L_{\rm np}=2Q_{n,m}/\tau_n\),
\[
    \mathbb E[L_{\rm np}^2\mid D]
    =
    \frac{4}{\tau_n^2}\mathbb E[Q_{n,m}^2\mid D]
    =
    \frac{4q_{1-\alpha}^2}{\tau_n^2}
    +
    O\!\left(\frac{1}{\tau_n^2B}\right).
\]
\end{proof}


Let
\(
q_{1-\alpha}
\)
denote the target non-private resampling quantile for the pivot
\(\tau_n d(\widehat\theta_n,\theta)\). Thus the leading non-private confidence
radius is \(q_{1-\alpha}/\tau_n\), and the true squared length is
\(
 \frac{4q_{1-\alpha}^2}{\tau_n^2}.
\)
 Under standard regularity conditions on the sampling distribution (see Lemma~\ref{lem:quantile_finiteB}) the expected squared length of the nonprivate confidence region is
\[
\mathbb \E \!\left[r(\alpha)^2\right]
=
\E\frac{4(Q_{n,m}^2+\sigma_m^2)}{\tau_n^2}=\frac{4q_{1-\alpha}^2}{\tau_n^2}+O\bb{\frac{1}{B\tau_n^2}}.
\]

~\cite{Dette2025GaussianDP}'s squared confidence
interval length is approximately
\[
\mathrm{Len}_{\mathrm{DG}}(\alpha)^2
\approx
\frac{4q_{1-\alpha}^2}{\tau_n^2}
+O\bb{\frac{1}{B\tau_n^2}}+
\frac{4q_{1-\alpha}^2}{\tau_n^2}
\cdot
\frac{\tau_m^2 B\ell^2}{n^2\mu^2}.
\]

\textbf{Our method.}
Algorithm~\ref{alg:gdp_quantile_subsample} forms the empirical resampling quantile
\(Q_{n,m}\), privatizes it once,
\[
\widetilde Q_{n,m}
=
\max\{Q_{n,m}+\sigma_m Z,0\},
\qquad
Z\sim N(0,1),
\]
and reports
\(
\widetilde\theta_n
\pm
\frac{\widetilde Q_{n,m}}{\tau_n}.
\)
For the comparison here, we isolate the contribution of the noisy quantile.

Moreover, for mean-like statistics with global sensitivity
\(\Delta_m\) on samples of size \(m\),
\(
\sigma_m
=
\frac{\tau_m\kappa_\star \Delta_m}{\mu},
\)
up to universal constants using Theorem~\ref{thm:global_guarantee}. Therefore
\[
\mathbb E_Z\!\left[\mathrm{Len}_{\mathrm{ours}}(\alpha)^2\right]
\approx
\frac{4q_{1-\alpha}^2}{\tau_n^2}
+O\bb{\frac{1}{B\tau_n^2}}+
\frac{4\tau_m^2\kappa_\star^2\Delta_m^2}
{\tau_n^2\mu^2}.
\]

Since \(\kappa_\star=O(\log(Bn))\) under with-replacement sampling and
\(\kappa_\star=1\) under sampling without replacement, our procedure replaces
the \(\sqrt B\)-type accumulation from repeatedly privatizing resample
statistics by a logarithmic dependence on \(B\) in the squared-length penalty.

In the root-\(n\) regime, where \(\tau_n=\sqrt n\), \(\tau_m=\sqrt m\), and
\(\Delta_m=\ell/m\), the privacy terms become
\[
\Delta_{\mathrm{DG}}
\approx
\frac{4q_{1-\alpha}^2Bm\ell^2}{n^3\mu^2},
\qquad
\Delta_{\mathrm{ours}}
\approx
\frac{4\kappa_\star^2\ell^2}{mn\mu^2}.
\]
Thus, for \(m\) close to \(n\), the prior approach pays essentially a factor
linear in \(B\), whereas our approach pays only through \(\kappa_\star^2\).

\paragraph{Chadha et al. \citep{chadha2024resampling}.}
Consider a mean-like statistic with
\(
    \Delta_n \le \frac{\ell}{n}\) and 
\(\tau_n=\sqrt n.
\)
The length calculation for~\cite{chadha2024resampling} is less direct because
their method privatizes resampling-based coverage estimates through a private
threshold search. We will use $B$ for their number of Monte Carlo resamples $N_{\rm mc}$.

Let \(t_\star\) be the smallest radius for which the empirical coverage
function satisfies \(p(t_\star)\ge 1-\alpha\), and let \(\widehat t\) be the
radius returned by their private threshold search. Their guarantee implies,
with high probability,
\[
1-\alpha-\omega_{\mathrm{cdf}}
\le
p(\widehat t)
\le
1-\alpha+2\omega_{\mathrm{cdf}}.
\]
If \(p\) is differentiable near \(t_\star\) and \(p'(t)\) is bounded away from $0$ in a neighborhood of \(t_\star\), then
\[
|\widehat t-t_\star|
=O\bb{
\omega_{\mathrm{cdf}}}=\tilde{O}\!\left(
    \frac{1}{\sqrt n}
    +
    \frac{\ell}{\varepsilon \sqrt n}
    +
    \sqrt{\frac{1}{N_{\rm mc}}}
\right)..
\]
 Thus
\[
    \widehat t
    =
    q_{1-\alpha}
    +
    \widetilde O\!\left(
        \frac{1}{\sqrt n}
        +
        \frac{\ell}{\varepsilon\sqrt n}
        +
        \sqrt{\frac{1}{N_{\rm mc}}}
    \right).
\]
Since \(L_{\rm Chadha}=2\widehat t/\tau_n\),
\bas{
    L_{\rm Chadha}^2
    &=
    \frac{4q_{1-\alpha}^2}{\tau_n^2}
    +
    \frac{1}{\tau_n^2} \tilde{O} \left(
        q_{1-\alpha}\left(
            \frac{\ell}{\varepsilon\sqrt n}
            +
            \sqrt{\frac{1}{N_{\rm mc}}}
    \right) + \left(
            \frac{\ell}{\varepsilon\sqrt n}
            +
            \sqrt{\frac{1}{N_{\rm mc}}}
    \right)^2 \right)
}
and the private error for the squared radius is of the order $\tilde{O}\bb{\frac{\ell}{\eps n^{3/2}} + \frac{1}{n \sqrt{N_{\rm mc}}}}$.

%% file: CRGS-SGD.tex

\subsection{Confidence regions for averaged SGD}
\label{sec:asgd}

Let $K\subset\mathbb R^d$ be a public compact convex set with nonempty
interior. Suppose that $\ell(\cdot;x)$ is twice continuously differentiable
on a neighborhood of $K$ and, for every $x\in\mathcal X$ and $w\in K$,
\begin{equation}
\label{eq:asgd-loss}
\lambda I_d\preceq\nabla^2\ell(w;x)\preceq L I_d,
\qquad
\|\nabla\ell(w;x)\|_2\leq G,
\end{equation}
where $0<\lambda\leq L$ and $G>0$ are public constants.
For $D_N=(X_1,\ldots,X_N)$ and a permutation $\pi$ of $[N]$, define
\begin{equation}
\label{eq:asgd-estimator}
w_t=\Pi_K\!\left(
w_{t-1}-\eta_t\nabla\ell(w_{t-1};X_{\pi(t)})
\right),
\quad t\in[N],
\qquad
\widehat\theta_N=\theta(D_N;\pi):=\frac1N\sum_{t=1}^N w_t,
\end{equation}
where $w_0\in K$ is public. Take $\eta_t=c(t+t_0)^{-a}$, with $c>0$,
$a\in(1/2,1)$, and a fixed $t_0\geq1$ sufficiently large that
\begin{equation}
\label{eq:asgd-steps}
\eta_t\leq L^{-1},
\qquad
\eta_{t+1}^{-1}-\eta_t^{-1}\leq\lambda/2,
\qquad t\geq0.
\end{equation}
Each estimator evaluation uses an independent uniform permutation,
independent of the data, and restarts from the same $w_0$ and step schedule.
Thus, there is one pass through the input positions, including repeated
observations in a with-replacement resample.

\begin{lemma}[Global sensitivity of averaged SGD]
\label{lem:asgd-gs}
Under \eqref{eq:asgd-loss}--\eqref{eq:asgd-steps}, for every fixed
permutation $\pi$,
\[
\Delta_2\bigl(\theta(\cdot;\pi)\bigr)
\leq \Delta_N:=\frac{4G}{\lambda N}.
\]
In particular, conditional on the permutation, averaged SGD is mean-like
with $\ell_N=4G/\lambda$.
\end{lemma}

\begin{proof}
Fix neighboring datasets and use the same permutation. Suppose that the
changed record is visited at time $j$, and let
$\delta_t:=\|w_t-w_t'\|_2$.
The Hessian bounds and nonexpansiveness of $\Pi_K$ imply that each
unchanged update is $(1-\lambda\eta_t)$-Lipschitz. Hence
\[
\delta_t\leq(1-\lambda\eta_t)\delta_{t-1}
       +2G\eta_t\mathbf 1\{t=j\},
\qquad
\delta_t\leq2G\eta_j
\prod_{s=j+1}^t(1-\lambda\eta_s)
\quad(t\geq j).
\]
Set
$A_{j,N}:=\sum_{t=j}^N\prod_{s=j+1}^t(1-\lambda\eta_s)$.
Backward induction gives $A_{j,N}\leq2/(\lambda\eta_j)$:
$A_{N,N}=1\leq2/(\lambda\eta_N)$, and
\[
A_{j,N}=1+(1-\lambda\eta_{j+1})A_{j+1,N}
\leq\frac{2}{\lambda\eta_{j+1}}-1
\leq\frac{2}{\lambda\eta_j},
\]
where the last inequality uses \eqref{eq:asgd-steps}. Therefore,
\[
\|\theta(D_N;\pi)-\theta(D_N';\pi)\|_2
\leq\frac1N\sum_{t=1}^N\delta_t
\leq\frac{2G\eta_j}{N}A_{j,N}
\leq\frac{4G}{\lambda N}.
\]
\end{proof}

\begin{fact}[Asymptotic normality of averaged SGD; \cite{PJ92}]
\label{fact:asgd-clt}
Suppose that $d$ is fixed, $X_1,\ldots,X_N$ are i.i.d.\ from $P$, and
\eqref{eq:asgd-loss}--\eqref{eq:asgd-steps} hold. Assume that
$f(w):=\mathbb E_P[\ell(w;X)]$ is finite on $K$, its minimizer
$\theta:=\arg\min_{w\in K}f(w)$ lies in $\operatorname{int}(K)$, and
\[
\|\nabla^2\ell(u;x)-\nabla^2\ell(v;x)\|_{\mathrm{op}}
\leq M\|u-v\|_2
\qquad(u,v\in K,\ x\in\mathcal X)
\]
for a finite constant $M$. Let
\[
H:=\nabla^2 f(\theta),
\qquad
V:=\mathbb E_P\!\left[
\nabla\ell(\theta;X)\nabla\ell(\theta;X)^\top
\right],
\qquad
\Sigma:=H^{-1}VH^{-1}.
\]
If $V\succ0$, then
\[
\sqrt N(\widehat\theta_N-\theta)
\xrightarrow{d}\mathcal N(0,\Sigma).
\]
More precisely,
\begin{equation}
\label{eq:asgd-linearization}
\sqrt N(\widehat\theta_N-\theta)
=-H^{-1}\frac1{\sqrt N}\sum_{t=1}^N
\nabla\ell(\theta;X_t)+R_N,
\qquad
\mathbb E\|R_N\|_2
\leq C\bigl(N^{(a-1)/2}+N^{1/2-a}\bigr).
\end{equation}
The constant $C$ is uniform over distributions satisfying the same bounds
and a common positive lower bound on
$\operatorname{dist}(\theta,\partial K)$.
\end{fact}

\begin{proof}
Since a data-independent permutation preserves the i.i.d.\ law, take
$\pi$ to be the identity. Write $e_t=w_t-\theta$.
Nonexpansiveness, strong convexity, and \eqref{eq:asgd-steps} give
\[
\mathbb E\|e_t\|_2^2
\leq(1-\lambda\eta_t)\mathbb E\|e_{t-1}\|_2^2+G^2\eta_t^2
\leq C\eta_t,
\]
where the last inequality follows by induction.
Decompose
\[
\nabla\ell(w_{t-1};X_t)=He_{t-1}+g_t+\xi_t+r_t,
\quad
g_t:=\nabla\ell(\theta;X_t),
\quad
r_t:=\nabla f(w_{t-1})-He_{t-1}.
\]
Then $\xi_t$ is a martingale difference and
\[
\mathbb E\|\xi_t\|_2^2
\leq L^2\mathbb E\|e_{t-1}\|_2^2,
\qquad
\|r_t\|_2\leq(M/2)\|e_{t-1}\|_2^2.
\]
Let
$p_t:=\eta_t^{-1}(w_t-w_{t-1})
+\nabla\ell(w_{t-1};X_t)$
be the projection correction. If
$r_0:=\operatorname{dist}(\theta,\partial K)$,
then $\|p_t\|_2\leq G$, and, for all sufficiently large $t$,
$p_t=0$ whenever $\|e_{t-1}\|_2<r_0/2$. Consequently,
$\mathbb E\|p_t\|_2\leq C\eta_{t-1}$
for all sufficiently large $t$. It follows that
\[
\mathbb E\left\|
\frac1{\sqrt N}\sum_{t=1}^N\xi_t
\right\|_2
=O(N^{-a/2}),
\qquad
\frac1{\sqrt N}\sum_{t=1}^N
\mathbb E\bigl[\|r_t\|_2+\|p_t\|_2\bigr]
=O(N^{1/2-a}).
\]
Also, summation by parts gives
\[
\sum_{t=1}^N\frac{e_{t-1}-e_t}{\eta_t}
=\frac{e_0}{\eta_1}-\frac{e_N}{\eta_N}
+\sum_{t=1}^{N-1}
(\eta_{t+1}^{-1}-\eta_t^{-1})e_t,
\]
whose expected norm is $O(N^{a/2})$.
Summing the recursion
$e_t-e_{t-1}
=-\eta_t(He_{t-1}+g_t+\xi_t+r_t-p_t)$,
and using
$N^{-1/2}\sum_{t=1}^Ne_{t-1}
=\sqrt N(\widehat\theta_N-\theta)+(e_0-e_N)/\sqrt N$,
proves \eqref{eq:asgd-linearization}.
All bounds are uniform under the stated restrictions. Finally,
$\mathbb E g_t=0$, $\operatorname{Cov}(g_t)=V$, and the multivariate
central limit theorem and Slutsky's theorem give the result.
\end{proof}

\begin{lemma}[Private confidence regions for averaged SGD]
\label{lem:asgd-crgs}
Suppose that the assumptions of Fact~\ref{fact:asgd-clt} hold.
Fix $\alpha\in(0,1)$ and $\mu>0$. Run CR-GS (Algorithm~2) with
\[
\widehat\theta_n=\theta(D;\pi_0),
\qquad
\theta_i^\star=\theta(D[S_i];\pi_i),
\qquad
\tau_N=\sqrt N,
\qquad
\Delta_N=\frac{4G}{\lambda N}.
\]
Use $\kappa_\star=1$ for WOR and
$\kappa_\star=\kappa_\star(n,m,B,\rho)$ from equation~(4) for WR.
Then the algorithm is $\mu$-GDP. If
$m\to\infty$, $m/n\to0$, $B\to\infty$, $\rho=o_n(1)$, and
\begin{equation}
\label{eq:asgd-private-small}
\frac{G}{\lambda\mu}
\left(
\frac{\sqrt d+\sqrt{\log(1/\rho)}}{\sqrt n}
+\frac{\kappa_\star\sqrt{\log(1/\rho)}}{\sqrt m}
\right)=o_n(1),
\end{equation}
then, for either resampling mode,
\[
\mathbb P(\theta\in\mathcal C_{n,\alpha})
=1-\alpha+o_n(1),
\qquad
\mathcal C_{n,\alpha}
:=\left\{u\in\mathbb R^d:
\|u-\widetilde\theta_n\|_2
<\frac{\widetilde Q_{n,m}}{\sqrt n}
\right\}.
\]
Set $\mathcal C_{n,\alpha}=\varnothing$ if the algorithm returns $\perp$.
On return, its confidence radius satisfies
\[
\frac{\widetilde Q_{n,m}}{\sqrt n}
=\max\left\{
\frac{Q_{n,m}}{\sqrt n}
+\sigma_{\mathrm{priv}}Z_1,0
\right\},
\qquad Z_1\sim\mathcal N(0,1),
\]
where
\[
\sigma_{\mathrm{priv}}
=\frac{4\sqrt2G}{\lambda\mu}
\left(
\frac{\kappa_\star}{\sqrt{mn}}
+\frac{\sqrt m}{n^{3/2}}
\right)
=O\!\left(
\frac{G\kappa_\star}{\lambda\mu\sqrt{mn}}
\right).
\]
\end{lemma}

\begin{proof}
For privacy, fix the resampling indices and permutations. On an accepted
realization, changing one original record changes each resample in at most
$\kappa_\star$ positions. Lemma~\ref{lem:asgd-gs} and the
$\ell_\infty$-Lipschitz property of order statistics therefore give
\[
\Delta_2(Q_{n,m})
\leq\sqrt m(\kappa_\star\Delta_m+\Delta_n).
\]
The noise scales in Algorithm~2 are consequently
\begin{equation}
\label{eq:asgd-noise}
\nu_n=\frac{4\sqrt2G}{\lambda\mu n},
\qquad
\sigma_m=\frac{4\sqrt2G}{\lambda\mu}
\left(
\frac{\kappa_\star}{\sqrt m}+\frac{\sqrt m}{n}
\right).
\end{equation}
Proposition~3.7 and GDP composition imply $\mu$-GDP for the joint release.
The indices, permutations, and abort decision are data-independent, so
averaging over them preserves this guarantee.

We next verify resampling validity. Define all resamples before the
multiplicity check, and let
\[
L_{n,m}(t):=\mathbb P^\star\!\left(
\sqrt m\|\theta(D[S];\pi)-\theta\|_2\leq t
\mid D
\right),
\qquad
F(t):=\mathbb P(\|Z\|_2\leq t),
\quad Z\sim\mathcal N(0,\Sigma).
\]
Here $\mathbb P^\star$ averages over the resampling indices and permutation.
For WOR, $L_{n,m}(t)$ is a bounded symmetric U-statistic of order $m$,
with the kernel averaged over $\pi$. Averaging over disjoint blocks yields
\[
\mathbb E L_{n,m}(t)=F_m(t),
\qquad
\operatorname{Var}(L_{n,m}(t))
\leq\frac1{4\lfloor n/m\rfloor}.
\]
Here
$F_m(t):=\mathbb P(
\sqrt m\|\widehat\theta_m-\theta\|_2\leq t)$.
Fact~\ref{fact:asgd-clt}, $m/n\to0$, and continuity of $F$ therefore imply
$\sup_t|L_{n,m}(t)-F(t)|\xrightarrow{p}0$.

For WR, put
$f_n(w):=n^{-1}\sum_{j=1}^n\ell(w;X_j)$ and
$\vartheta_n:=\arg\min_{w\in K}f_n(w)$.
Strong convexity gives
\[
\|\vartheta_n-\theta\|_2
\leq\lambda^{-1}\|\nabla f_n(\theta)\|_2
=O_p(n^{-1/2}).
\]
Thus $\vartheta_n$ has an interior margin bounded away from zero with
probability tending to one. Moreover, by the law of large numbers and
the Lipschitz assumptions,
\[
H_n:=\nabla^2 f_n(\vartheta_n)\xrightarrow{p}H,
\qquad
V_n:=\frac1n\sum_{j=1}^n
\nabla\ell(\vartheta_n;X_j)
\nabla\ell(\vartheta_n;X_j)^\top
\xrightarrow{p}V.
\]
Conditional on $D$, the resampled observations are i.i.d.\ from the
empirical distribution. The uniform bound in
\eqref{eq:asgd-linearization}, applied to this distribution, and the
conditional Lindeberg central limit theorem apply: the scores have mean
zero on the interior event, are bounded by $G$, and have covariance $V_n$.
Hence the conditional law of
$\sqrt m(\theta(D[S];\pi)-\vartheta_n)$
converges in probability to $\mathcal N(0,\Sigma)$.
Since $\sqrt m\|\vartheta_n-\theta\|_2=o_p(1)$, this again gives
$\sup_t|L_{n,m}(t)-F(t)|\xrightarrow{p}0$.

For either mode, centering at $\widehat\theta_n$ instead of $\theta$
changes every resampling pivot by at most
$\sqrt m\|\widehat\theta_n-\theta\|_2
=O_p(\sqrt{m/n})=o_p(1)$.
Conditional on $D$ and $\pi_0$, the $B$ resampling pivots are independent.
The DKW inequality and $B\to\infty$ therefore imply
\[
Q_{n,m}\xrightarrow{p}
q_{1-\alpha}:=F^{-1}(1-\alpha).
\]
Since $\Sigma\succ0$, $F$ is continuous and strictly increasing on
$[0,\infty)$, and $q_{1-\alpha}>0$.
This verifies Assumption~3.8.

Substituting \eqref{eq:asgd-noise} into the definition of $\gamma$ in
Theorem~4.2, condition~\eqref{eq:asgd-private-small} gives $\gamma=o_n(1)$.
Gaussian concentration and the $1$-Lipschitz property of truncation imply
\[
\sqrt n\|\widetilde\theta_n-\widehat\theta_n\|_2=o_p(1),
\qquad
\widetilde Q_{n,m}-Q_{n,m}=o_p(1).
\]
Fact~\ref{fact:asgd-clt} and Lemma~B.4 now yield the claimed coverage
for the auxiliary outputs defined before the multiplicity check.
The actual outputs agree with them except on an event of probability
at most $\rho=o_n(1)$, proving the coverage statement for CR-GS.
Finally, the radius formula follows from
$\sigma_{\mathrm{priv}}=\sigma_m/\sqrt n$.
\end{proof}